\documentclass[preprint,11pt]{article}

\usepackage{comment,url,algorithm,algorithmic,graphicx,subfigure}
\usepackage{amssymb,amsfonts,amsmath,amsthm,amscd,dsfont,mathrsfs}
\usepackage{float,psfrag,epsfig,color,url,hyperref}
\usepackage{epstopdf,bbm,mathtools,enumitem}

\makeatletter
\@addtoreset{equation}{section}
\makeatother

\newtheorem{theorem}{Theorem}[section]

\newtheorem{lemma}[theorem]{Lemma}
\newtheorem{assumption}{Assumption}

\newtheorem{corollary}[theorem]{Corollary}

\newtheorem{remark}{Remark}

\definecolor{Red}{rgb}{1,0,0}
\definecolor{Blue}{rgb}{0,0,1}
\definecolor{Olive}{rgb}{0.41,0.55,0.13}
\definecolor{Green}{rgb}{0,1,0}
\definecolor{MGreen}{rgb}{0,0.8,0}
\definecolor{DGreen}{rgb}{0,0.55,0}
\definecolor{Yellow}{rgb}{1,1,0}
\definecolor{Cyan}{rgb}{0,1,1}
\definecolor{Magenta}{rgb}{1,0,1}
\definecolor{Orange}{rgb}{1,.5,0}
\definecolor{Violet}{rgb}{.5,0,.5}
\definecolor{Purple}{rgb}{.75,0,.25}
\definecolor{Brown}{rgb}{.75,.5,.25}
\definecolor{Grey}{rgb}{.5,.5,.5}
\definecolor{Pink}{rgb}{1,0,1}
\definecolor{DBrown}{rgb}{.5,.34,.16}
\definecolor{Black}{rgb}{0,0,0}

\def\ind{\mathbbm{1}}

\def\proj{\mathcal{P}}
\def\tol{{\epsilon}}
\def\step{{\eta}}
\def\dphi{{\Phi^{(1)}}}
\def\ddphi{{\Phi^{(2)}}}

\def\X{\bold{X}}
\def\Q{\bold{Q}}
\def\W{\bold{W}}
\def\X{\bold{X}}

\def\H{\bold{H}}
\def\St{S_t}
\def\I{\bold{I}}

\def\hth{\hat{\theta}}
\def\th{{\theta}}

\def\tth{{\tilde{\th}}}

\def\e{{\epsilon}}
\def\diam{{{\rm diam}(\C)}}
\def\T{{\mathcal T}}
\def\C{{\mathcal C}}
\def\D{{\mathcal D}}

\def\reals{{\mathbb R}}

\def\<{\langle}
\def\>{\rangle}
\def\E{{\mathbb E}}
\def\P{{\mathbb P}}
\def\grad{{\nabla}}
\def\gradth{{\nabla_\th}}
\def\gradS{\boldsymbol{\nabla}^2}
\def\gradSth{\boldsymbol{\nabla}_\th^2}

\def\lmax{\lambda_{\text{max}}}
\def\lmin{\lambda_{\text{min}}}

\def\O{{\mathcal O}}

\def\argmin{{\rm argmin}}

\def\ALG{NewSamp }

\newcommand{\eq}[1]{\begin{alignat}{3}#1\end{alignat}}
\newcommand{\eqn}[1]{\begin{alignat*}{3}#1\end{alignat*}}
\newcommand{\commentout}[1]{}

\begin{document}

\title{Convergence rates of sub-sampled Newton methods
}

\author{
Murat A. Erdogdu\thanks{Department of Statistics,
Stanford University}\ \ \ \ \ \ \ \ \ \ \ 
Andrea Montanari${}^{*,}$\thanks{Department of Electrical Engineering, 
Stanford University}
}

\date{}

\maketitle

\begin{abstract}
We consider the problem of minimizing a sum of $n$ functions 
via projected iterations onto a convex parameter set $\C \subset \reals^p$ where $n\gg p\gg 1$. 
In this regime, algorithms which utilize sub-sampling techniques are known to be effective.
In this paper, we use sub-sampling techniques together with 
eigenvalue thresholding to design a new randomized batch algorithm 
which possesses comparable convergence rate to Newton's method, 
yet has much smaller per-iteration cost. 
The proposed algorithm is robust in terms of starting point and 
step size, and enjoys a composite convergence rate, namely, 
quadratic convergence at start and linear convergence 
when the iterate is close to the minimizer. 
We develop its theoretical analysis which also allows us to 
select near-optimal algorithm parameters. 
Our theoretical results can be used to obtain 
convergence rates of previously proposed sub-sampling based algorithms as well. 
We demonstrate how our results apply to well-known machine learning problems.
Lastly, we evaluate the performance of our algorithm on several datasets 
under various scenarios. 
\end{abstract}

\section{Introduction}
We consider the problem of minimizing  an average of $n$ functions $f_i : \reals^p \to \reals$,
\eq{
\underset{\th}{\text{minimize}} \ f(\th) \coloneqq \frac{1}{n}\sum_{i=1}^n f_i (\th)\, ,
}
in a batch setting, where $n$ is assumed to be much larger than $p$.
Most machine learning models can be expressed as above, 
where each function $f_i$ corresponds to an observation. 
Examples include logistic regression, support vector machines, 
neural networks and graphical models. 

Many optimization algorithms have been developed
to solve the above minimization problem using iterative methods
\cite{bishop1995neural,Boyd:2004,nesterov2004introductory}. 
In this paper, we consider the iterations of the following form 
\eq{\label{eq::update}
\th^{t+1} = \th^t -\step_t \Q^t \gradth f(\th^t),
}
where $\step_t$ is the step size and $\Q^t$ is a suitable scaling matrix 
that provides curvature information (For simplicity, we drop the projection throughout the introduction, i.e., we assume $\C=\reals^p$).

Updates of the form Eq.~(\ref{eq::update}) have been 
extensively studied in the optimization literature. 
The case where $\Q^t$ is equal  to the identity matrix corresponds to
\emph{Gradient Descent} (GD) which,
under smoothness assumptions, achieves linear convergence 
rate
with $\O(np)$ per-iteration cost. 
More precisely, GD with ideal step size yields
$$
\|\hth^{t+1}-\theta_*\|_2\le \xi^t_{1,\mbox{\tiny GD}}
\|\hth^{t}-\theta_*\|_2\, ,
$$
where, as $\lim_{t\to\infty}\xi^t_{1,\mbox{\tiny GD}} =
1-(\lambda_p^*/\lambda_1^*)$, and $\lambda^*_i$ is the $i$-th 
largest eigenvalue of the Hessian of $f(\theta)$ at
minimizer $\theta_*$.

Second order methods such as \emph{Newton's Method} 
(NM) and \emph{Natural Gradient Descent} (NGD) \cite{amari1998natural} 
can be recovered by taking $\Q^t$ to be the inverse Hessian and
the Fisher information evaluated at the current iterate, respectively.
Such methods may achieve quadratic convergence 
rates with $\O(np^2+p^3)$ per-iteration cost
\cite{bishop1995neural,nesterov2004introductory}.
In particular, for $t$ large enough, Newton's Method yields  $$\|\hth^{t+1}-\theta_*\|_2\le \xi^t_{2,\text{NM}}
\|\hth^{t}-\theta_*\|_{2}^2,$$ and it is insensitive to the condition
number of the Hessian.
However, when the 
number of samples grows large, computation of $\Q^t$
becomes extremely expensive.

A popular line of research tries to 
construct the matrix $\Q^t$ in a way that the update is computationally 
feasible, yet still provides sufficient second order information.
Such attempts resulted in Quasi-Newton methods, in which only
gradients and iterates are used in the construction of matrix $\Q^t$,
resulting in an efficient update at each step $t$.
A celebrated Quasi-Newton method is the 
\emph{Broyden-Fletcher-Goldfarb-Shanno} (BFGS) algorithm
\cite{broyden1970convergence,fletcher1970new,goldfarb1970family,shanno1970conditioning}
which requires $\O(np+p^2)$ per-iteration cost
\cite{bishop1995neural,nesterov2004introductory}.

An alternative approach  is
to use \emph{sub-sampling} techniques, where scaling matrix $\Q^t$ is 
based on randomly selected set of data points \cite{martens2010deep,byrd2011use,VinyalsAISTATS12}. 
Sub-sampling is widely used in the first order methods, but is not as well
studied for approximating the scaling matrix. In particular, 
 theoretical guarantees are still missing.
 
A key challenge is that the sub-sampled Hessian  is 
close to the actual Hessian along the directions corresponding to large
eigenvalues (large curvature directions in $f(\theta)$), but is a poor
approximation in the directions corresponding to small eigenvalues
(flatter directions in $f(\theta)$). In order to overcome this
problem, we use low-rank approximation. More precisely, we treat all the
eigenvalues below the $r$-th as if they were equal to the $(r+1)$-th.
This yields the desired stability with respect to the sub-sample:
we call our algorithm \ALG\!\!. In
this paper, we establish the following:

\begin{enumerate}
\item \ALG has a composite convergence rate: quadratic at start
  and linear near the minimizer, as illustrated in
  Figure \ref{fig::convAndCoeff}.
Formally, we prove a bound of the form $$\|\hth^{t+1}-\theta_*\|_2\le
\xi_1^t\|\hth^{t}-\theta_*\|_2+ \xi_2^t\|\hth^{t}-\theta_*\|^2_2$$ with
coefficient that are explicitly given (and are computable from data).
\item The asymptiotic behavior of the linear convergence coefficient
  is $\lim_{t\to\infty} \xi_1^t =1 -
  (\lambda^*_{p}/\lambda^*_{r+1})+\delta$, for $\delta$ small. The condition number 
  $(\lambda^*_{1}/\lambda^*_{p})$ which controls the convergence of
  GD, has been replaced by the milder
  $(\lambda^*_{r+1}/\lambda^*_{p})$.
For datasets with strong spectral features, this can be a large
improvement, as shown in Figure~\ref{fig::convAndCoeff}.
\item The above results are achived without tuning the step-size, in
particular, by setting $\step_t=1$.
\item The complexity per iteration of \ALG is $\O(np+|S|p^2)$
with $|S|$ the sample size.
\item Our theoretical results can be used to obtain 
convergence rates of previously proposed sub-sampling
algorithms. 

\end{enumerate}
We demonstrate the performance of \ALG on four datasets, 
and compare it to the well-known optimization
methods.

The rest of the paper is organized as follows: 
Section~\ref{sec:relatedWork} surveys the related work. 
In Section~\ref{sec:algorithm}, we describe the proposed algorithm and 
provide the intuition behind it. 
Next, we present our theoretical results in Section~\ref{sec::theory}, i.e., 
convergence rates corresponding to different sub-sampling schemes,
followed by a discussion on how to choose the algorithm parameters.
Two applications of the algorithm are discussed in Section~\ref{sec::examples}. 
We compare our algorithm with several existing methods on 
various datasets in Section~\ref{sec::experiments}. 
Finally, in Section~\ref{sec::discussion}, we conclude with a brief discussion.

\subsection{Related Work}\label{sec:relatedWork}
Even a synthetic review of optimization algorithms for large-scale
machine learning would go beyond the page limits of this paper. 
Here, we emphasize that 
the method of choice depends crucially on
the amount of data to be used, and their dimensionality (i.e.,
respectively, on the parameters $n$ and $p$).
In this paper, we focus on a regime in which $p$ is large but not so large as to
make matrix manipulations (of order $p^2$ to $p^3$) impossible. 
Also $n$ is large but not so large as to make batch gradient
computation (of order $np$) prohibitive. On the other hand, our aim is
to avoid $\O(np^2)$ calculations required by standard Newton
method. Examples of this regime are given in Section \ref{sec::examples}.

In contrast, online algorithms are the option of choice for very large
$n$ since the computation per update is independent of $n$. 
 In the case of \emph{Stochastic Gradient Descent} (SGD), 
the descent direction is formed by a randomly selected gradient 
\cite{robbins1951stochastic}. Improvements to SGD have been developed
by incorporating the previous gradient directions in the current update 
\cite{schmidt2013minimizing,senior2013empirical,bottou2010large,Duchi11}.

Batch algorithms, on the other hand, can achieve faster
convergence and exploit second order information. They are competitive
for intermediate $n$. Several methods in this category aim at
quadratic, or at least super-linear convergence rates. In particular, 
Quasi-Newton methods have  proven  effective 
\cite{bishop1995neural,nesterov2004introductory}. 
Another approach towards the same goal is to utilize sub-sampling to 
form an approximate Hessian \cite{martens2010deep,byrd2011use,VinyalsAISTATS12,qu2015sdna,erdogdu2015convergence,erdogdu2015newton-stein}. 
If the sub-sampled Hessian is close to the true Hessian, these methods 
can approach NM in terms of convergence rate, nevertheless, they enjoy 
much smaller complexity per update.
No convergence rate analysis is available for these methods: this
analysis  is
the main contribution
of our paper.  To the best of our knowledge, the best result in this
direction is
proven in \cite{byrd2011use} that estabilishes asymptotic convergence
without quantitative bounds (exploiting general theory from \cite{griva2009linear}).

Further improvements have been suggested either by utilizing \emph{Conjugate Gradient} (CG) methods
 and/or using Krylov sub-spaces \cite{martens2010deep,byrd2011use,VinyalsAISTATS12}. 
Sub-sampling can be also used to obtain an approximate solution,
if an exact solution is not required \cite{Dhillon-etal2013}.
Lastly, there are various hybrid algorithms that combine two or more 
techniques to gain improvement. 
Examples include, sub-sampling and Quasi-Newton \cite{schraudolph2007stochastic,sohl2013adaptive,byrd2014stochastic}, 
SGD and GD \cite{friedlander2012hybrid}, 
NGD and NM \cite{roux2010fast}, NGD and low-rank approximation \cite{roux2008topmoumoute}. 


\setlength{\textfloatsep}{10pt}
 \begin{algorithm}[t]
\caption{\ALG \label{alg::newsamp}}
\begin{algorithmic}
    \STATE {\bfseries Input:} $\hth^0, r, \tol, \{\step_t, |S_t|\}_{t},t=0$.
    	\begin{enumerate}
        \STATE {\bfseries Define:} $\proj_\C(\th)= \argmin_{\th' \in \C} \| \th - \th' \|_2$ is the Euclidean projection onto $\C$,\\
        
        \ \ \ \ \ \ \ \ \ \ \ \   $ [\mathbf{U}_k,\mathbf{\Lambda}_k] = \text{TruncatedSVD}_k(\H)$ is the rank-$k$ truncated SVD of $\H$ with $(\mathbf{\Lambda}_k)_{ii}=\lambda_i$.\\
        \WHILE{$\|\hth^{t+1}- \hth^{t}\|_2 \leq \tol$}
               \STATE Sub-sample a set of indices $\St\subset [n]$.
               \STATE Let $\H_{\St} = \frac{1}{|S_t|}\sum_{i\in \St}\gradSth f_i(\hth^t)$, \ \ \ and \ \ \, $ [\mathbf{U}_{r+1},\mathbf{\Lambda}_{r+1}] = \text{TruncatedSVD}_{r+1}(\H_{\St} )$,
               \STATE $\Q^t =  \lambda_{r+1}^{-1}\I_p +  \mathbf{U}_{r}\left(\mathbf{\Lambda}_r^{-1} -\lambda_{r+1}^{-1}\I_r\right) \mathbf{U}_{r}^T$ ,
        \STATE $\hth^{t+1} = \proj_{\C} \left(\hth^{t} - \step_t \Q^t \grad_\th f(\hth^{t})\right)$,
	\STATE $t \leftarrow t+1$.
         \ENDWHILE
         \end{enumerate}
         \STATE {\bfseries Output:} $\hth^{t}$.
\end{algorithmic}

\end{algorithm}
%
%

\section{\ALG\!: A Newton method via sub-sampling and eigenvalue thresholding} \label{sec:algorithm}

In the regime we consider, $n \gg p \gg 1$, there are 
two main drawbacks associated with the classical second order methods such as Newton's method. 
The predominant issue in this regime is the computation of the Hessian matrix, 
which requires $\O(np^2)$ operations, 
and the other
issue is finding the inverse of the Hessian, 
which requires $\O(p^3)$ 
computation. 
Sub-sampling is an effective and efficient way of addressing the first issue, 
by forming an approximate Hessian to exploit curvature information.
Recent empirical studies show that sub-sampling the
Hessian provides significant improvement in terms of computational cost, yet 
preserves the fast convergence rate of second order methods \cite{martens2010deep,VinyalsAISTATS12,erdogdu2015newton-stein-long}. 
If a uniform sub-sample is used, 
the sub-sampled Hessian will be a random matrix with expected value at the true Hessian,
which can be considered as a sample estimator to the mean.
Recent advances in statistics have shown that the performance of various estimators can be
significantly improved by simple procedures such as \emph{shrinkage}  and/or \emph{thresholding}
\cite{cai2010singular,donoho2013optimal,gavish2014optimal,gavish2014optimal}. 
To this extent, we use a specialized low-rank approximation as the important 
second order information is generally contained in the largest few 
eigenvalues/vectors of the Hessian.
We will see in Section \ref{sec::theory}, 
how this procedure provides faster convergence rates 
compared to the bare sub-sampling methods.
 
\ALG is presented as Algorithm \ref{alg::newsamp}. 
At iteration step $t$, 
the sub-sampled set of indices, its size and the corresponding sub-sampled 
Hessian is denoted by $S_t$, $|S_t|$ and $\H_{\St}$, respectively.
Assuming that the functions $f_i$'s are convex, eigenvalues of the symmetric matrix $\H_{\St}$ are non-negative. Therefore, singular value (SVD) and eigenvalue decompositions coincide.
The operation $ \text{TruncatedSVD}_{k}(\H_{\St} ) =  [\mathbf{U}_{k},\mathbf{\Lambda}_{k}] $ is the best rank-$k$ approximation, i.e., 
takes $\H_{\St}$ as input and returns the largest $k$ eigenvalues in the diagonal matrix $\mathbf{\Lambda}_{k} \in \reals^{k \times k}$ with the corresponding $k$ eigenvectors $\mathbf{U}_{k} \in \reals^{p \times k}$. 
This procedure requires $\O(kp^2)$ computation using a standard method, though there are faster randomized algorithms which provide accurate approximations to the truncated SVD problem with much less computational cost \cite{halko2011finding}. 
To construct the curvature matrix $[\Q^t]^{-1}$, instead of using the basic rank-$r$ approximation, 
we fill its 0 eigenvalues with the $(r+1)$-th eigenvalue of the sub-sampled Hessian which is
the largest eigenvalue below the threshold. If we compute a truncated SVD with $k=r+1$ and $(\mathbf{\Lambda}_{k})_{ii} = \lambda_i$, the described operation can be formulated as the following,
\eq{\label{eq::matrixQinv}
\Q^t =  \lambda_{r+1}^{-1}\I_p +  \mathbf{U}_{r}\left(\mathbf{\Lambda}_r^{-1} -\lambda_{r+1}^{-1}\I_r\right) \mathbf{U}_{r}^T,
}
which is simply the sum of a scaled identity matrix and a rank-$r$ matrix. 
Note that the low-rank approximation that is
suggested to improve the curvature estimation has been further utilized to reduce the cost
of computing the inverse matrix.
%
Final per-iteration cost of \ALG will be 
$\O\left(np+(|S_t|+r)p^2\right)\approx\O\left(np+|S_t|p^2\right)$.
\ALG takes the parameters $\{\step_t,|S_t|\}_t$ and $r$ as inputs. 
We discuss in Section \ref{sec::parameters}, how to choose these parameters near-optimally,
based on the theory we develop in Section \ref{sec::theory}. 

Operator $\proj_\C$ projects the current
iterate to the feasible set $\C$ using Euclidean projection. 
Throughout, we assume that this projection can be done efficiently. 
In general, most unconstrained optimization problems do not require this step, and can be omitted. The purpose of projected iterations in our algorithm is mostly theoretical, and will be clear in Section \ref{sec::theory}.

By the construction of $\Q^t$, \ALG will always be a descent algorithm.
It enjoys a quadratic convergence rate at start which 
transitions into a linear rate in the neighborhood of the minimizer. 
This behavior can be observed in Figure \ref{fig::convAndCoeff}.
The left plot in Figure 1 shows the convergence behavior of \ALG 
over different sub-sample sizes.
We observe that large sub-samples result in better convergence rates as expected. 
As the sub-sample size increases, slope of the linear phase
decreases, getting closer to that of quadratic phase at the transition point.
This phenomenon will be explained in detail in Section \ref{sec::theory}, by 
Theorems \ref{thm::uniformSampling} and \ref{thm::generalSampling}.
The right plot in Figure \ref{fig::convAndCoeff} demonstrates
how the coefficients of linear and quadratic phases depend on the thresholded rank.
Note that the coefficient of the quadratic phase increases with the rank threshold, 
whereas for the linear phase, relation is reversed.

\begin{figure*}[t]
\centering
 \includegraphics[width=3.2in]{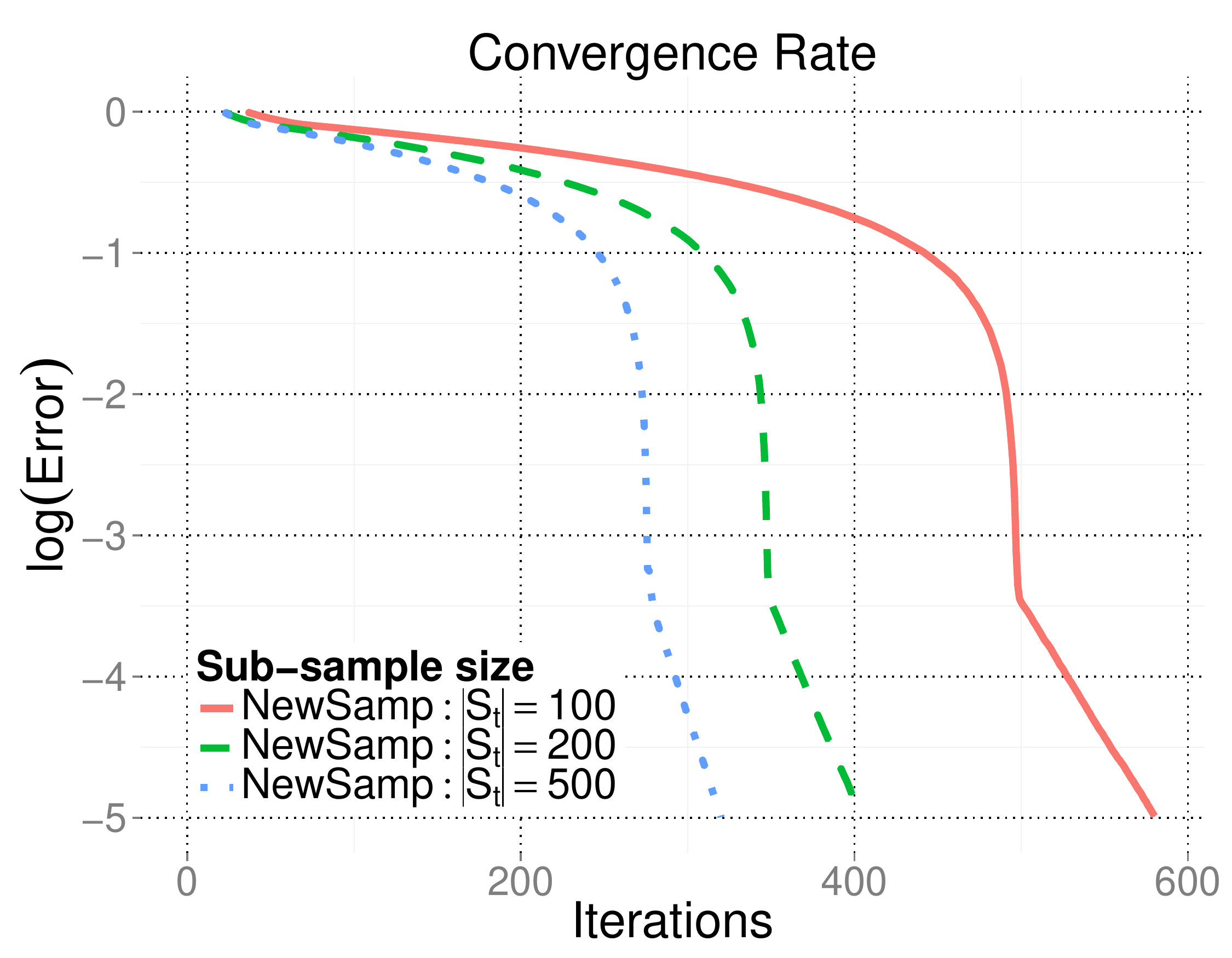}
 \includegraphics[width=3.2in]{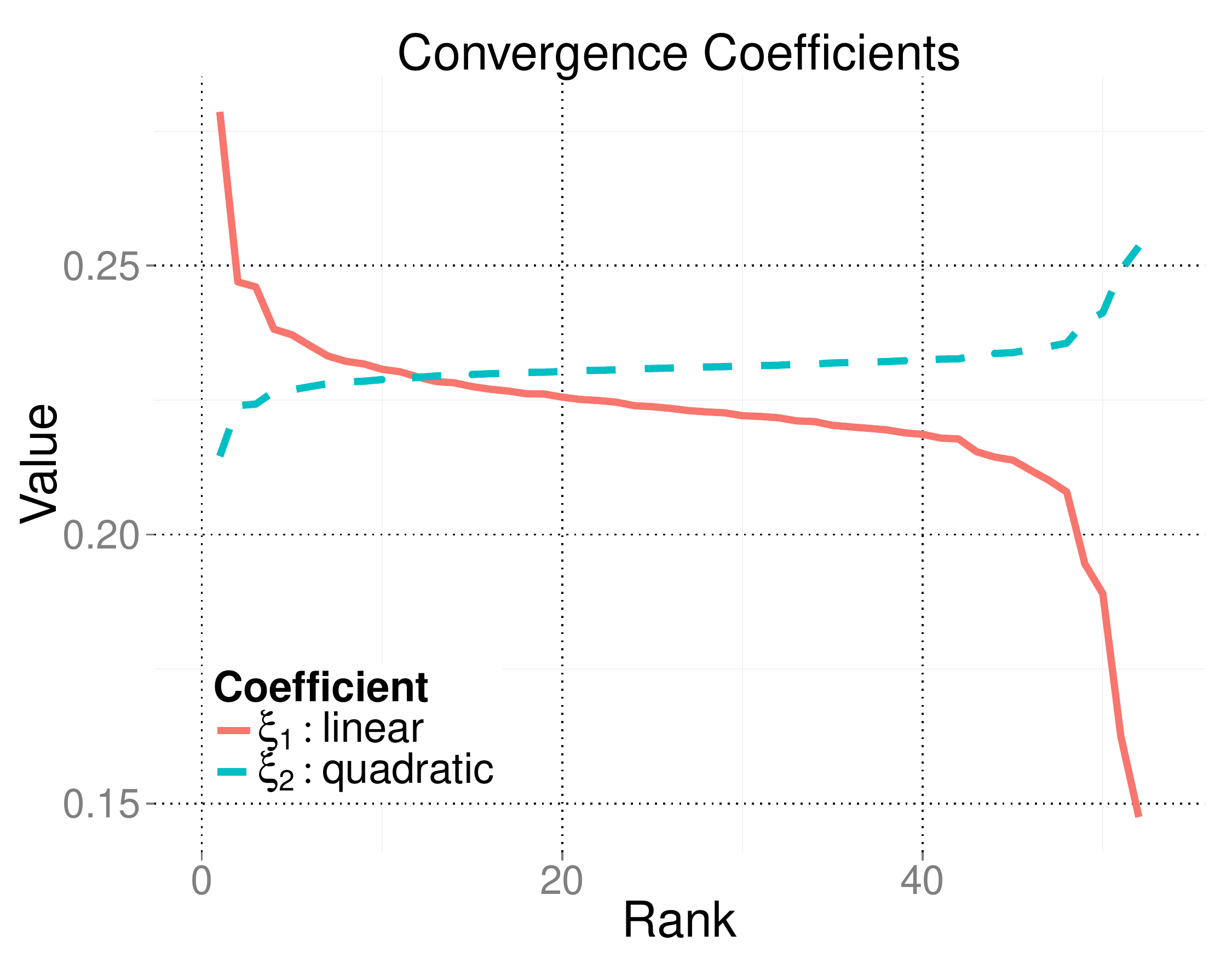}
  \caption{\label{fig::convAndCoeff}
Left plot demonstrates convergence rate of \ALG, which starts with a quadratic rate and transitions into linear convergence near the true minimizer. 
The right plot shows the effect of eigenvalue thresholding on the convergence coefficients. $x$-axis shows the number of kept eigenvalues. 
Plots are obtained using \emph{Covertype} dataset.
}
\end{figure*}

\section{Theoretical results}\label{sec::theory}

In this section, we provide the convergence analysis of \ALG based on two different
sub-sampling schemes:

\begin{itemize}
\item[S1:] \textbf{Independent sub-sampling}: 
At each iteration $t$, $S_t$ is uniformly sampled from $[n] = \{ 1,2,...,n\}$, 
independently from the sets $\{S_\tau\}_{\tau<t}$, 
with or without replacement.
\item[S2:] \textbf{Sequentially dependent sub-sampling}: 
At each iteration $t$, $S_t$ is sampled from $[n]$, 
based on a distribution which might depend on the previous sets 
$\{S_\tau\}_{\tau<t}$, but not on any randomness in the data.
\end{itemize}
The first sub-sampling scheme is simple and commonly used in optimization. 
One drawback is that the sub-sampled set at the current iteration is 
independent of the previous sub-samples, 
hence does not consider which of the samples were 
previously used to form the approximate curvature information. 
In order to prevent cycles and obtain better performance near the optimum,
one might want to increase the sample size as the iteration advances 
\cite{martens2010deep}, including previously unused samples. 
This process results in a sequence of dependent sub-samples 
which falls into the sub-sampling scheme S2.
In our theoretical analysis, we make the following assumptions:

\begin{assumption}[Lipschitz continuity] \label{as::Lipschitz}
For any subset $S \subset [n]$, there exists a constant $M_{|S|}$ 
depending on the size of $S$, such that $\forall \th, \th' \in \C$,
\eqn{
 \left\|\H_S(\th)- \H_{S}(\th') \right\|_2  \leq M_{|S|}\  \| \th - \th' \|_2.
}
\end{assumption}

\begin{assumption}[Bounded Hessian]\label{as::bound}
$\forall i =  1,2,...,n$, the Hessian of the function $f_i(\th)$, $\gradSth f_i(\th)$, is upper bounded by an absolute constant $K$, i.e.,
$$
\max\limits_{i\leq n}\left\| \gradSth f_i (\th)\right\|_2 \leq K.
$$
\end{assumption}

\subsection{Independent sub-sampling}

In this section, we assume that $S_t\subset [n]$ is sampled according to the sub-sampling scheme S1. 
In fact, many stochastic algorithms assume that $S_t$ is 
a uniform subset of $[n]$, because 
in this case the sub-sampled Hessian is 
an unbiased estimator of the full Hessian. 
That is, $\forall \th \in \C$,
$
\E \left [\H_{S_t}(\th) \right ] = \H_{[n]}(\th),
$
where the expectation is over the randomness in $S_t$.
We next show that for any scaling matrix $\Q^t$ that is formed by the sub-samples $S_t$,
iterations of the form Eq.~(\ref{eq::update}) will have a composite convergence rate, 
i.e., combination of a linear and a quadratic phases.
\begin{lemma}\label{lem::uniformSampling}
Assume that the parameter set $\C$ is convex and 
$S_t\subset [n]$ is based on sub-sampling scheme S1.
Further, let the Assumptions \ref{as::Lipschitz} and \ref{as::bound} hold and $\th_* \in \C$.
Then, for an absolute constant $c>0$, with probability at least $1-2/p$,
the updates of the form Eq.~(\ref{eq::update}) satisfy
\eqn{
\| \hth^{t+1} -\th_* \|_2 \leq&\ 
\xi_1^t \|\hth^{t} -\th_*\|_2
+\xi_2^t \|\hth^{t} -\th_*\|^2_2 ,
}
for coefficients $\xi_1^t$ and $\xi_2^t$ defined as
\eqn{
\xi_1^t = 
\left\| I - \step_t\Q^t\H_{S_t}(\hth^t)\right \|_2  
+  \step_t cK
 \left\|\Q^t\right\|_2 \sqrt{\frac{\log(p)}{|S_t|}}, \hspace{.8in}
 \xi_2^t = 
  \step_t
 \frac{ M_n}{2}
 \left\|\Q^t\right\|_2 .
}
\end{lemma}
\begin{remark}
If the initial point $\hth^0$ is close to $\th_*$, the algorithm will start with 
a quadratic rate of convergence which will transform into linear 
rate later in the close neighborhood of the optimum. 
\end{remark}
The above lemma holds for any matrix $\Q^t$.
In particular, if we choose $\Q^t = \H_{S_t}^{-1}$,
we obtain a bound for the simple sub-sampled Hessian method.
In this case, the coefficients 
$\xi_1^t$ and $\xi_2^t$ depend on $\| \Q^t\|_2 = 1/\lambda^t_p$ where
$\lambda^t_p$ is the smallest eigenvalue of the sub-sampled Hessian.
Note that $\lambda^t_p$ can be arbitrarily small which might blow up 
both of the coefficients. In the following, we will see how \ALG remedies this issue.

\begin{theorem}\label{thm::uniformSampling}
Let the assumptions in Lemma \ref{lem::uniformSampling} hold.
Denote by $\lambda^t_i$, the $i$-th eigenvalue of $\H_{S_t}(\hth^t)$
where $\hth^t$ is given by \ALG at iteration step $t$.
If the step size satisfies
\eq{\label{eq::stepCond}
\step_t \leq \frac{2}{1+{\lambda^t_p}/{\lambda^t_{r+1}}},
}
then we have, 
with probability at least $1-2/p$,
\eqn{
\| \hth^{t+1} -\th_* \|_2 \leq&\ 
\xi_1^t \|\hth^{t} -\th_*\|_2
+\xi_2^t \|\hth^{t} -\th_*\|^2_2 ,
}
for an absolute constant $c>0$, for the coefficients $\xi_1^t$ and $\xi_2^t$ are defined as
\eqn{
\xi_1^t = 
1-\step_t \frac{\lambda^t_p}{\lambda^t_{r+1}}
+ \step_t\frac{cK}
{\lambda^t_{r+1}} \sqrt{\frac{\log(p)}{|S_t|}}, \hspace{.8in}
 \xi_2^t = 
\step_t\frac{ M_n }
{2\lambda^t_{r+1}}
.
}
\end{theorem}

\ALG has a composite convergence rate where $\xi_1^t$ and $\xi_2^t$ are 
the coefficients of the linear and the quadratic terms, respectively (See the right plot in Figure \ref{fig::convAndCoeff}). 
We observe that the sub-sampling size has a significant effect on the linear term, 
whereas the quadratic term is governed by the Lipschitz constant. 
We emphasize that the case $\step_t=1$ is feasible for the conditions of Theorem \ref{thm::uniformSampling}.
In the case of quadratic functions, since the Lipschitz constant is 0 , we obtain $\xi_2^t=0$ 
and the algorithm converges linearly. Following corollary summarizes this case.

\begin{corollary}[Quadratic functions]\label{cor::quadUniform}
Let the assumptions of Theorem \ref{thm::uniformSampling} hold. Further, assume that
 $\forall i \in [n]$,
the functions $\th:\reals^p \to f_i(\th)$ are quadratic.
Then, for $\hth^t$ given by \ALG at iteration step $t$, 
for the coefficient $\xi^t_1$ defined as in Theorem \ref{thm::uniformSampling}, 
with probability at least $1-2/p$, 
we have
\eq{\label{eq::quadRate}
\| \hth^{t+1} -\th_* \|_2 \leq&\ 
\xi_1^t \|\hth^{t} -\th_*\|_2
.
}
\end{corollary}
%
%

\subsection{Sequentially dependent sub-sampling}\label{sec::dependent}

Here, we assume that the sub-sampling scheme S2 
is used to generate $\{ S_\tau\}_{\tau\geq1}$. Distribution of
sub-sampled sets may depend on each other, but not on any randomness in the dataset.
Examples include fixed sub-samples as well as 
sub-samples of increasing size, sequentially covering unused data.
In addition to Assumptions \ref{as::Lipschitz}-\ref{as::bound}, we assume the following.
\begin{assumption}[i.i.d. observations]\label{as::iidFunctions}
Let $z_1,z_2,...,z_n \in \mathrm{Z}$ be i.i.d. observations from a distribution $\D$. 
For a fixed $\th \in \reals^p$ and $\forall i \in [n]$, we assume that the functions $\{f_i\}_{i=1}^n$ satisfy 
$
f_i(\th) = \varphi(z_i,\th),
$
 for some function $\varphi : \mathrm{Z}\times\reals^p \to \reals$. 
\end{assumption}
Most statistical learning algorithms can be formulated as above, e.g., in classification problems, 
one has access to i.i.d. samples $\{(y_i,x_i)\}_{i=1}^n$ where $y_i$ and $x_i$ 
denote the class label and the covariate, and $\varphi$ measures the 
classification error (See Section \ref{sec::examples} for examples).
For the sub-sampling scheme S2, an analogue of 
Lemma \ref{lem::uniformSampling} is stated in Appendix as Lemma \ref{lem::generalSampling}, which immediately leads to the following theorem.
\begin{theorem}\label{thm::generalSampling}
Assume that the parameter set $\C$ is convex and 
$S_t\subset [n]$ is based on the sub-sampling scheme S2.
Further, let the Assumptions \ref{as::Lipschitz}, \ref{as::bound} and 
\ref{as::iidFunctions} hold, almost surely.
Conditioned on the event $\mathcal{E} = \{ \th_* \in \C \}$,
if the step size satisfies Eq.~\ref{eq::stepCond},
then
for $\hth^t$ given by \ALG at iteration $t$, 
with probability at least $1-c_\mathcal{E}\, e^{-p}$ for $c_\mathcal{E} = c/\P(\mathcal{E})$, we have

\eqn{
\| \hth^{t+1} -\th_* \|_2 \leq&\ 
\xi_1^t \|\hth^{t} -\th_*\|_2
+\xi_2^t \|\hth^{t} -\th_*\|^2_2 ,
}
for the coefficients $\xi_1^t$ and $\xi_2^t$ defined as
\eqn{
\xi_1^t& = 
 1-\step_t \frac{\lambda_p^t}{\lambda_{r+1}^t}
+ \step_t\frac{c'K}
{\lambda_{r+1}^t} 
\sqrt{
  \frac{p}{|S_t|} 
\log 
\bigg(
\frac{  \diam^2 \left ( M_n + M_{|S_t|}\right)^2 |S_t|}{K^2}
\bigg)
}
, \hspace{.6in}
 \xi_2^t &= 
\step_t
\frac{    M_n}
{2\lambda_{r+1}^t}
,
}
where $c,c'>0$ are absolute constants and $\lambda_i^t$ denotes the $i$-th eigenvalue of $\H_{S_t}(\hth^t)$.
\end{theorem}
%
%
Compared to the Theorem \ref{thm::uniformSampling}, 
we observe that the coefficient of the quadratic
term does not change. This is due to Assumption \ref{as::Lipschitz}. 
However, the bound on the linear term is worse, 
since we use the uniform bound over the convex parameter set $\C$. 
The same order of magnitude is also observed by \cite{erdogdu2015newton-stein-long}, 
which relies on a similar proof technique.
Similar to Corollary \ref{cor::quadUniform}, we have the following result for the quadratic functions.
\begin{corollary}[Quadratic functions]\label{cor::quadGeneral}
Let the assumptions of Theorem \ref{thm::generalSampling} hold. Further assume that
 $\forall i \in [n]$,
the functions $\th \to f_i(\th)$ are quadratic.
Then, conditioned on the event $\mathcal{E}$, with probability at least $1-c_\mathcal{E}\, e^{-p}$, \ALG iterates satisfy
\eqn{
\| \hth^{t+1} -\th_* \|_2 \leq&\ 
\xi_1^t \|\hth^{t} -\th_*\|_2
,
}
for coefficient $\xi_1^t$ defined as in Theorem \ref{thm::generalSampling}.
\end{corollary}

\subsection{Dependence of coefficients on $t$ and convergence guarantees}

The coefficients $\xi_1^t$ and $\xi_2^t$ depend on the iteration step 
which is an undesirable aspect of the above results. However, these 
constants can be well approximated by their analogues $\xi_1^*$ and $\xi_2^*$ 
evaluated at the optimum which are defined by simply replacing $\lambda_j^t$ 
with $\lambda_j^*$ in their definition, where the latter is the $j$-th eigenvalue of 
full-Hessian at $\theta_*$.
For the sake of simplicity, we only consider the case where the functions $\th \to f_i(\th)$ are quadratic.
\begin{theorem}\label{thm::coefficients}
Assume that the functions $ f_i(\th)$ are quadratic,
$S_t$ is based on scheme S1 and $\step_t=1$.
Let the full Hessian at $\th_*$ be lower bounded by a constant $k$. 
Then for sufficiently large $|S_t|$, we have, with probability $1-2/p$
\eqn{
\left | \xi_1^t - \xi_1^*\right| \leq 
\frac{c_1K \sqrt{{\log(p)}/{|S_t|}}}
{k\big(k- c_2K\sqrt{{\log(p)}/{|S_t|}}\big) }\coloneqq \delta,
}
for some absolute constants $c_1,c_2$.
\end{theorem}

Theorem \ref{thm::coefficients} implies that,
when the sub-sampling size is sufficiently large, 
$\xi_1^t$ will concentrate around $\xi_1^*$. 
Generalizing the above theorem to non-quadratic functions is straightforward, 
in which case, one would get additional terms involving the difference $\| \hth^t - \th_*\|_2$.
In the case of scheme S2, if one uses fixed sub-samples, i.e., $\forall t$, $S_t = S$,
then the coefficient $\xi_1^t$ does not depend on $t$. 
The following corollary gives a sufficient condition for convergence. 
A detailed discussion on the number of iterations 
until convergence and further local convergence properties
can be found in Appendix \ref{sec::theoryExtra}.

\begin{corollary}\label{cor::convergenceCondition}
Assume that $\xi_1^t$ and $\xi_2^t$ are well-approximated by 
$\xi_1^*$ and $\xi_2^*$ with an error bound of $\delta$, i.e., $\xi_i^t \leq \xi_i^* +\delta$ for $i=1,2$, as in 
Theorem \ref{thm::coefficients}.
For the initial point $\hth^{0}$, a sufficient condition for convergence is 
\eqn{
\| \hth^{0} -\th_* \|_2 <\frac{1-\xi_1^*-\delta}{\xi_2^* +\delta}.
}
\end{corollary}

\subsection{Choosing the algorithm parameters}\label{sec::parameters}

Algorithm parameters play a crucial role in
most optimization methods.
Based on the theoretical results from previous sections,
we discuss procedures to choose the optimal values for
the step size $\step_t$, sub-sample size $|S_t|$ and rank threshold.

\begin{itemize}
\item
\emph{Step size:} For the step size of \ALG at iteration $t$, we suggest
\eq{\label{eq::step1}
\step_t(\gamma) = \frac{2}{1+{\lambda^t_p}/{\lambda^t_{r+1}}+\gamma}.
}
where $\gamma=\O(\log(p)/|S_t|)$.
Note that $\step_t(0)$ is the upper bound in Theorems \ref{thm::uniformSampling} and \ref{thm::generalSampling} and it minimizes the first component
of $\xi_1^t$. 
The other terms in $\xi_1^t$ and $\xi^t_2$ linearly depend on $\step_t$.
To compensate for that, we shrink $\step_t(0)$ towards 1. 
Contrary to most algorithms, optimal step size of \ALG is larger than 1.
See Appendix \ref{sec::stepsize} for a rigorous derivation of Eq.~\ref{eq::step1}.

\item
\emph{Sample size:} By Theorem \ref{thm::uniformSampling}, 
a sub-sample of size $\O((K/\lambda^*_p)^2\log(p))$ should be sufficient
to obtain a small coefficient for the linear phase. 
Also note that sub-sample size $|S_t|$ scales quadratically with the condition number.

\item
\emph{Rank threshold:} For a full-Hessian with effective rank $R$ 
(trace divided by the largest eigenvalue),
it suffices to use $\O(R\log(p))$ samples
\cite{vershynin2010introduction,vershynin2012close}. 
Effective rank is upper bounded by the dimension $p$. 
Hence, one can use $p\log(p)$ samples to approximate the 
full-Hessian and choose a rank threshold which retains the important curvature information.
%
%
\end{itemize}

\section{Examples}\label{sec::examples}

\subsection{Generalized Linear Models}
Finding the maximum likelihood estimator in Generalized Linear Models (GLMs)
is equivalent to minimizing the negative log-likelihood $f(\th)$,
\eq{\label{eq::glmLikelihood}
\underset{\th }{\text{minimize}} \ f(\th)  = \frac{1}{n}\sum_{i=1}^n\left [  \Phi (\<x_i,\th\>)-y_i\<x_i , \th \> \right ],
}
where $\Phi$ is the \emph{cumulant generating function}, $y_i\in \reals$ denotes the observations, $x_i\in \reals^p$ denotes the rows of design matrix $\X \in \reals^{n \times p}$, and $\th\in\reals^{p}$ is the coefficient vector.
Note that this formulation only considers GLMs with canonical links.
Here, $\<x , \th \>$ denotes the inner product between 
the vectors $x$, $\th$.
The function $\Phi$ defines the type of GLM. Well known examples include
ordinary least squares (OLS) with $\Phi(z)=z^2$,
logistic regression (LR) with $\Phi(z)=\log(1+e^z)$,
and Poisson regression (PR) with $\Phi(z)=e^z$.

The gradient and the Hessian of the 
above function can be written as:
\eq{
\grad_\th f (\th) = \frac{1}{n}\sum_{i=1}^n\left [  \dphi (\<x_i,\th\>)x_i - y_ix_i \right ], \ \ \ \ 
\gradS_\th f (\th) = \frac{1}{n}\sum_{i=1}^n \ddphi (\<x_i,\th\>)x_i x_i^T.
}
We note that the Hessian of the GLM problem is always positive definite. This is because the second derivative of the cumulant generating function is simply the variance of the observations. Using the results from Section \ref{sec::theory}, we perform a
convergence analysis of our algorithm on a GLM problem.

\begin{corollary}\label{cor::glm}
Let $S_t\subset [n]$ be a uniform sub-sample, and $\C$ be a convex parameter set.
Assume that the second derivative of the cumulant generating function, $\ddphi$ is
 bounded by $1$, and it is Lipschitz continuous with 
Lipschitz constant $L$. Further, assume that the covariates 
are contained in a ball of radius $\sqrt{R_x}$, 
i.e. $\max_{i \in [n]} \| x_i \|_2 \leq \sqrt{R_x}.$
Then,
for $\hth^t$ given by \ALG with constant step size $\step_t =1$ at iteration $t$,
with probability at least $1-2/p$, we have
\eqn{
\| \hth^{t+1} -\th_* \|_2 \leq&\ 
\xi^t_1 \|\hth^{t} -\th_*\|_2
+\xi^t_2 \|\hth^{t} -\th_*\|^2_2 ,
}
for constants $\xi^t_1$ and $\xi^t_2$ defined as
\eqn{
\xi^t_1 = &
1- \frac{\lambda^t_i}{\lambda^t_{r+1}}
+  \frac{cR_x}
{\lambda^t_{r+1}} \sqrt{\frac{\log(p)}{|S_t|}}, \ \ \ \ \ \ \ \ \ \ \ \ \ \ \ \ \ \ \ \ 
 \xi^t_2 = &
\frac{L R_x^{3/2} }
{2\lambda^t_{r+1}}
,
}
where $c>0$ is an absolute constant and $\lambda^t_i$ is the $i$th eigenvalue of \ $\H_{S_t}(\hth^t)$.
\end{corollary}
Proof of Corollary \ref{cor::glm} can be found in Appendix \ref{sec::proofs}. Note that the bound on the second derivative is quite loose for Poisson regression due to exponentially fast growing cumulant generating function.
%
%
%

\subsection{Support Vector Machines}

A linear Support Vector Machine (SVM) provides a \emph{separating hyperplane} which maximizes the \emph{margin}, i.e., 
the distance between the hyperplane and the support vectors. 
Although the vast majority of the
literature focuses on the dual problem 
\cite{vapnik1998statistical,scholkopf2002learning}, 
SVMs can be trained using the primal as well. 
Since the dual problem does not scale well with the number of data points (some approaches get $\O(n^3)$ complexity, 
\cite{woodsend2011exploiting}), 
the primal might be better-suited for optimization of 
linear SVMs \cite{keerthi2005modified,chapelle2007training}. 

The primal problem for the linear SVM can be written as
\eq{\label{eq::svmPrimal}
\underset{\th \in \C}{\text{minimize}} \ f(\th) = \frac{1}{2}\|\th\|_2^2 
+\frac{1}{2} C \sum_{i =1}^n \ell(y_i,\<\th,x_i\>)%
}
where $(y_i,x_i)$ denote the data samples, $\th$ defines the separating hyperplane, $C>0$ and $\ell$ could be any loss function.
The most commonly used loss functions include \emph{Hinge-p loss}, 
\emph{Huber loss} and their smoothed versions \cite{chapelle2007training}. 
Smoothing or approximating such losses with more stable functions is sometimes crucial in optimization. In the case of \ALG which requires the loss function to be twice differentiable (almost everywhere), we suggest either smoothed Huber loss, 
i.e.,
\eqn{
\ell(y,\<\th,x\>) =
\begin{cases}
    0, & \text{if $y\<\th,x\> >3/2$},\\
     \frac{(3/2-y\<\th,x\>)^2}{2}, & \text{if $|1-y\<\th,x\>|\leq1/2$},\\
    1-y\<\th,x\>, & \text{otherwise}.
  \end{cases}
}
or Hinge-2 loss, i.e., $$\ell(y,\<\th,x\>) = \max \left\{  0, 1-y \< \th, x\>\right\}^2.$$
For the sake of simplicity, we will focus on Hinge-2 loss.
Denote by $SV_t$, the set of indices of all the support vectors at iteration $t$, i.e.,
\eqn{
SV_t = \{i : y_i \< \th^t,x_i\> <1  \}.
}
When the loss is set to be the Hinge-2 loss, the Hessian of the SVM problem, 
normalized by the number of support vectors, can be written as

\eqn{
\gradS_\th f(\th) =\frac{1}{|SV_t|} \Big\{\I +  C \sum_{i \in SV_t} x_ix_i^T\Big\}.
}
When $|SV_t|$ is large, 
the problem falls into our setup and can be solved efficiently using \ALG\!.
Note that unlike the GLM setting, Lipschitz condition of our Theorems do not apply here. However, we empirically demonstrate that \ALG works regardless of such assumptions. 

\section{Experiments}\label{sec::experiments}

\begin{figure*}[t]
\centering
 \includegraphics[width=6.9in]{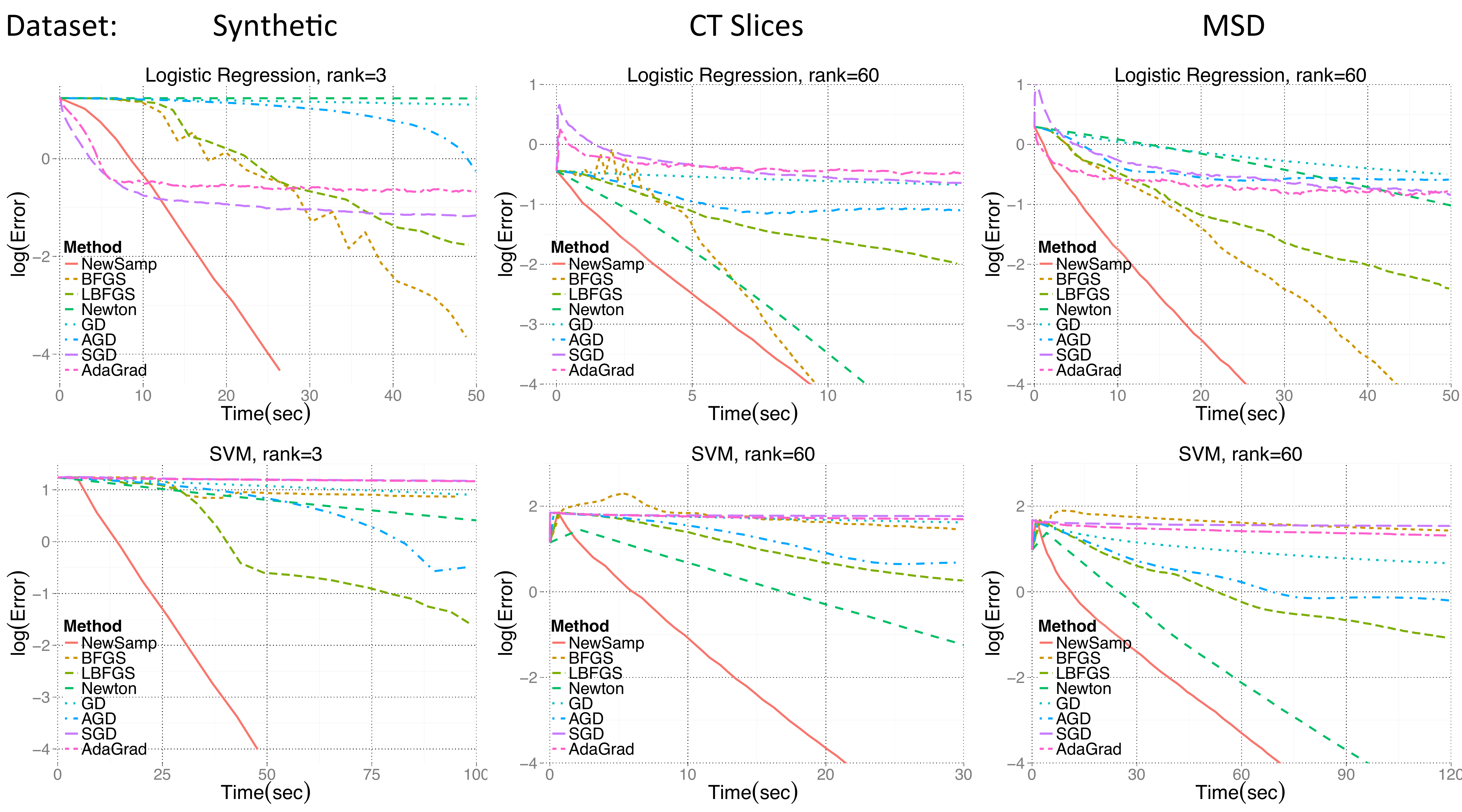}
  \caption{\label{fig::plot}
Performance of various optimization methods on different datasets.
\ALG is represented with red color .
}
\end{figure*}
In this section, we validate the performance of \ALG through 
extensive numerical studies. 
We experimented on two optimization problems, namely,
\emph{Logistic Regression} (LR) and \emph{Support Vector Machines} (SVM) with quadratic loss.
LR minimizes Eq.~\ref{eq::glmLikelihood} for the logistic function, whereas
SVM minimizes Eq.~\ref{eq::svmPrimal} for the Hinge-2 loss.

In the following, we briefly describe the algorithms that are used in the experiments:
\begin{enumerate}
\item \emph{Gradient Descent} (GD), at each iteration, takes a step proportional to negative of the full gradient evaluated at the current iterate. Under certain regularity conditions, GD exhibits a linear convergence rate.
\item \emph{Accelerated Gradient Descent} (AGD) is proposed by Nesterov \cite{nesterov1983method}, which improves over the gradient descent by using a momentum term. 
Performance of AGD strongly depends of the smoothness of the function $f$ and decreasing step size adjustments may be necessary for convergence.

\item \emph{Newton's Method} (NM) achieves a quadratic convergence rate by utilizing the inverse Hessian evaluated at the current iterate. However, the computation of Hessian makes it impractical for large-scale datasets.
\item \emph{Broyden-Fletcher-Goldfarb-Shanno} (BFGS) is the most popular and stable Quasi-Newton method. Scaling matrix is formed by accumulating the information from iterates and gradients, satisfying \emph{Quasi-Newton rule}. 
The convergence rate is locally super-linear and per-iteration cost is comparable to first order methods.
\item \emph{Limited Memory BFGS} (L-BFGS) is a variant of BFGS, which uses only the recent iterates and gradients to form the approximate Hessian, providing significant improvement in terms of memory usage. 
\item \emph{Stochastic Gradient Descent} (SGD) is a simplified version of GD where, at each iteration, instead of the full gradient, a randomly selected gradient is used. 
Per-iteration cost is independent of $n$, 
yet the convergence rate is significantly slower compared to batch algorithms. 
We follow the guidelines of \cite{bottou2010large,senior2013empirical} for the step size,
, i.e.,
$$
\gamma_t = \frac{\gamma}{1+t/c},
$$
 for constants $\gamma,c>0$.
\item \emph{Adaptive Gradient Scaling} (AdaGrad) is an online algorithm which uses an adaptive learning rate based on the previous gradients. AdaGrad significantly improves the performance and stability of SGD \cite{Duchi11}. 
This is achieved by scaling each entry of gradient differently.
, i.e., at iteration step $t$, step size for the $j$-th coordinate is
 \eqn{
 (\gamma_t)_j = \frac{\gamma}{
 \sqrt{
 \delta + \sum_{\tau=1}^t (\grad_\th f(\hth^t))_j}
 },
 }
for constants $\delta,\gamma>0$.
\end{enumerate}
For each of the batch algorithms, we used constant step size, and for all the algorithms, 
we choose the step size that provides the fastest convergence. 
For the stochastic algorithms, we optimized over the parameters that define the step size. 
Parameters of \ALG are selected following the guidelines described in Section \ref{sec::parameters}.

We experimented over various datasets that are given in Table \ref{tab::realdatasets}. 
The real datasets are downloaded from the UCI repository \cite{lichman2013}.
Each dataset consists of a design matrix $\X \in \reals^{n \times p}$ and 
the corresponding observations (classes) $y \in \reals^n$. Synthetic data is 
generated through a multivariate Gaussian distribution with a randomly generated covariance matrix.
As a methodological choice, we selected moderate values of $p$, for which Newton's Method can still be implemented, and nevertheless we can demonstrate an improvement. For larger values of $p$, comparison is even more favorable to our approach.

The effects of sub-sampling size $|S_t|$ and rank threshold are demonstrated in Figure \ref{fig::convAndCoeff}.
A thorough comparison of the aforementioned optimization techniques is presented 
in Figure \ref{fig::plot}. 
In the case of LR, we observe that stochastic algorithms enjoy fast convergence at start, 
but slows down later as they get close to the true minimizer. 
The algorithm that comes close to \ALG in terms of performance is BFGS. 
In the case of SVM, Newton's method is the closest algorithm to NewSamp, yet in all scenarios, \ALG
outperforms its competitors.
Note that the global convergence of BFGS is not better than that of GD \cite{nesterov2004introductory}. The condition for super-linear rate is $\sum_t \|\theta^t-\theta_*\|_2 < \infty$ for which, an initial point close to the optimum is required \cite{dennis1977quasi}. 
This condition can be rarely satisfied in practice, which also affects the performance of the other second order methods.  
For \ALG, even though the rank thresholding provides a certain level of robustness, we observed that the choice of a good starting point is still an important factor.
Details about Figure \ref{fig::plot} can be found in Table \ref{tab::details} in Appendix. For additional experiments and a detailed discussion,
see Appendix \ref{sec::furtherExperiments}.

\begin{table}[H]
\centering
\begin{tabular}{|l||l|l|l|r|}
\hline
Dataset    & $n$ & $p$ & $r$&  Reference \\
\hline
CT slices		&53500	&386	& 60&\cite{graf20112d}\\
Covertype 	&581012 	&54	 & 20&\cite{blackard1999comparative}\\
MSD  	&515345	&90	 & 60&\cite{bertin2011}\\
Synthetic  	&500000	&300	& 3& -\\
\hline
\end{tabular}
\caption{\label{tab::realdatasets}
Datasets used in the experiments.
}
\end{table}

\section{Conclusion}\label{sec::discussion}

In this paper, we proposed a sub-sampling based 
second order method utilizing low-rank Hessian estimation. 
The proposed method has the target regime $n\gg p$ and 
has $\O \left(np+|S|p^2\right)$ complexity per-iteration. 
We showed that the
convergence rate of \ALG is composite for two widely used sub-sampling schemes, i.e., 
starts as quadratic convergence and transforms to 
linear convergence near the optimum. 
Convergence behavior under other sub-sampling schemes is an interesting line of research.
Numerical experiments on both real and synthetic datasets 
demonstrate the performance of the proposed algorithm 
which we compared to the classical optimization methods.

\section*{Acknowledgments}
We are grateful to Mohsen Bayati for stimulating conversations on the topic of this work.
We would like to thank Robert M. Gower for carefully reading this manuscript and providing valuable feedback.
A.M. was partially supported by NSF grants CCF-1319979 and DMS-1106627 and the AFOSR
grant FA9550-13-1-0036. 


\bibliographystyle{amsalpha}
\bibliography{bib}

\newpage
\appendix

\section{Proofs of Theorems and Lemmas}\label{sec::proofs}

\begin{proof}[Proof of Lemma \ref{lem::uniformSampling}]
We write,
\eqn{
 \hth^{t} -\th_*- \step_t \Q^t \grad_\th f(\hth^{t})
=& \ \hth^{t} -\th_*- \step_t \Q^t \int_0^1\gradS_\th f(\th_* + \tau(\hth^{t}-\th_*))(\hth^{t} -\th_*)\ d\tau,  \\
=& \left( I - \step_t \Q^t \int_0^1\gradS_\th f(\th_* + \tau(\hth^{t}-\th_*))d\tau\right) (\hth^{t} -\th_*)\,.
}
Since the projection $\proj_{\C}$ in step 2 of \ALG can only 
decrease the $\ell_2$ distance, we obtain
\eqn{
\| \hth^{t+1} -\th_* \|_2 \leq&
 \left\| I - \step_t \Q^t 
 \int_0^1\gradS_\th f (\th_* + \tau(\hth^{t}-\th_*))d\tau\right\|_2 \|\hth^{t} -\th_*\|_2 .
}
Note that the first term on the right hand side governs 
the convergence behavior of the algorithm.

Next, for an index set $S \subset [n]$, define the matrix $\H_S(\th)$ as
\eqn{
\H_S(\th) = \frac{1}{|S|}\sum_{i\in S}\H_i (\th)
}
where $|S|$ denotes the size of the set. Denote the integral in the above equation by $\widetilde{\H}$,
that is,
\eqn{
\widetilde{\H} =  \int_0^1\gradS_\th f (\th_* + \tau(\hth^{t}-\th_*))d\tau .
}

By the triangle inequality, the governing term that determines 
the convergence rate can be bounded as
\eq{\label{eq::mainBound}
\left\| I -  \step_t\Q^t \widetilde{\H}\right\|_2&\leq
 \left\| I - \step_t\Q^t\H_S(\hth^t)\right \|_2 \\ 
 + &
 \step_t\left\|\Q^t\right\|_2 
 \Big\{  
 \left\|\H_S(\hth^t)- \H_{[n]}(\hth^t) \right\|_2 
+ \left\|\H_{[n]}(\hth^t)- \widetilde{\H} \right\|_2 
 \Big\},\nonumber
}
which holds, regardless of the choice of $\Q^t$.

In the following, we will use some matrix concentration results to bound the right hand side of Eq.~(\ref{eq::mainBound}). The result for sampling with replacement can be obtained by matrix Hoeffding's inequality given in \cite{tropp2012user}. Note that this explicitly assumes that the samples are independent.
For the concentration bounds under sampling without replacement (see i.e. \cite{gross2010note,gross2011recovering,mackey2014matrix}),
we will use the Operator-Bernstein inequality given in \cite{gross2010note} which is provided in Section \ref{sec::auxiliary} as Lemma \ref{lem::op-Bernstein} for convenience.

Using any indexing over the elements of sub-sample $S$, we denote the each element in $S$ by $s_i$, i.e.,
$$S=\{s_1, s_2, ... ,s_{|S|} \}.$$ 
For $\th \in \C$, we define the centered Hessians, $\W_i(\th)$ as
\eqn{
\W_i(\th) = \H_{s_i}(\th) -\E\left [ \H_{s_i}(\th) \right],
}
where the $\E\left [ \H_{s_i}(\th) \right]$ is just the full Hessian at $\th$.

By the Assumption (\ref{as::bound}), we have
\eq{\label{eq::boundsHoeff}
&\max_{i \leq n}\|\H_i(\th)\|_2 = \left\| \gradSth f_i (\th)\right\|_2 \leq K, \\
&\max_{i \leq n}\|\W_i\|_2 \leq 2K \coloneqq \gamma,\ \ \ \ \ \ 
\max_{i \leq n}\left\|  \W_i^2\right\|_2 \leq 4K^2\coloneqq \sigma^2.\nonumber
}
Next, we apply the matrix Bernstein's inequality given in Lemma \ref{lem::op-Bernstein}. For $\e \leq 4K$, and $\th \in \C$,

\eq{\label{eq::hoeffding}
\P \left(
\left\|\H_S(\th)- \H_{[n]}(\th) \right\|_2 > \e
\right)
\leq
2p \exp \left \{ 
-\frac{\e^2 |S|}{16 K^2}
\right\}.
}

Therefore, to obtain a convergence rate of $\O(1/p)$, we let 
$$\e = 
C\sqrt{\frac{ \log(p)}{|S|}},
$$
where $C = 6K$ is sufficient. We also note that the condition on $\e$ is trivially satisfied by 
our choice of $\e$ in the target regime.

For the last term, we may write,

\begin{align*}
\left\|\H_{[n]}(\hat{\theta}^t) - 
\widetilde{\H}
 \right\|_2
= &
\left\|\H_{[n]}(\hat{\theta}^t) - 
\int_0^1\gradS_\theta f (\theta_* + \tau(\hat{\theta}^{t}-\theta_*))d\tau
 \right\|_2, \\ \leq &
 \int_0^1\left\|\H_{[n]}(\hat{\theta}^t) - 
\gradS_\theta f (\theta_* + \tau(\hat{\theta}^{t}-\theta_*))
 \right\|_2 d\tau, \\
 \leq &\int_0^1 M_n (1-\tau)\|\hat{\theta}^t - \theta_* \|_2 d\tau, \\
  = &\frac{M_n}{2} \|\hat{\theta}^t - \theta_* \|_2 .
\end{align*}

First inequality follows from the fact that norm of an integral is less than or equal to the integral of the norm.
Second inequality follows from the Lipschitz property.

Combining the above results, we obtain the following 
for the governing term in Eq.(\ref{eq::mainBound}):
For some absolute constants $c,C>0$, 
with probability at least $1-2/p$, 
we have
\eqn{
\left\| I -  \step_t\Q^t \H_{[n]}(\tth^t)\right\|_2&\leq
 \left\| I - \step_t\Q^t\H_S(\hth^t)\right \|_2  
 + &
 \step_t
 \left\|\Q^t\right\|_2 
 \Big\{  
6K\sqrt{\frac{\log(p)}{|S|}}
+ \frac{M_{n}}{2}\  \| \hth^t - \th_* \|_2
 \Big\}.\nonumber
}
Hence, the proof is completed.
\end{proof}

\begin{proof}[Proof of Theorem \ref{thm::uniformSampling}]
Using the definition of $\Q^t$ in \ALG\!\!,
we immediately obtain that
\eq{\label{eq::governingTerm}
\left\| I - \step_t\Q^t\H_{S_t}(\hth^t)\right \|_2  =
\max_{i > r} \left\{ \left| 1-\step_t \frac{\lambda^t_i}{\lambda^t_{r+1}}\right | \right\},
}
and that $
 \left\|\Q^t\right\|_2 = { 1}/
{\lambda^t_{r+1}}$.
Then the proof follows from Lemma \ref{lem::uniformSampling} 
and by the assumption on the step size.
\end{proof}

\begin{lemma}\label{lem::generalSampling}
Assume that the parameter set $\C$ is convex and 
$S_t\subset [n]$ is based on sub-sampling scheme S2.
Further, let the Assumptions \ref{as::Lipschitz}, \ref{as::bound} and 
\ref{as::iidFunctions} hold, almost surely.
Then, for some absolute constants $c,C>0$, with probability at least $1-e^{-p}$,
the updates of the form stated in Eq.~(\ref{eq::update}) satisfy
\eqn{
\| \hth^{t+1} -\th_* \|_2 \leq&\ 
\xi_1^t \|\hth^{t} -\th_*\|_2
+\xi_2^t \|\hth^{t} -\th_*\|^2_2 ,
}
for coefficients $\xi_1^t,\xi_2^t$ defined as
\eqn{
\xi_1^t = &
\left\| I - \step_t\Q^t\H_{S_t}(\hth^t)\right \|_2  
+ \step_t 
 \left\|\Q^t\right\|_2 \times 
cK \sqrt{
  \frac{p}{|S_t|} 
\log\left (\frac{  \diam^2 \left ( M_n + M_{|S_t|}\right)^2 |S_t|}{K^2}\right)}
, \\
\xi_2^t = & 
  \step_t
  \frac{M_n}{2}
 \left\|\Q^t\right\|_2 .
}

\end{lemma}

\begin{proof}[Proof of Lemma \ref{lem::generalSampling}]
The first part of the proof is the same as 
Lemma \ref{lem::uniformSampling}. 
We carry our analysis from Eq.(\ref{eq::mainBound}). 
Note that in this general set-up, the iterates are random variables 
that depend on the random functions. 
Therefore, we use a uniform bound for the right hand side in 
Eq.(\ref{eq::mainBound}). That is,

\eqn{
\left\| I -  \step_t\Q^t \widetilde{\H}\right\|_2&\leq
 \left\| I - \step_t\Q^t\H_S(\hth^t)\right \|_2 \\ 
 + &
 \step_t\left\|\Q^t\right\|_2 
 \Big\{  
\sup_{\th \in \C} \left\|\H_S(\th)- \H_{[n]}(\th) \right\|_2 
+ \frac{M_n}{2} \|\hat{\theta}^t - \theta_* \|_2
 \Big\}.\nonumber
}

By the Assumption \ref{as::Lipschitz}, given $\th, \th' \in \C$ 
such that $\|\th - \th'\|_2 \leq \Delta$, we have,
\eqn{
  \left\|\H_S(\th)- \H_{[n]}(\th) \right\|_2  \leq &  \left\|\H_S(\th')- \H_{[n]}(\th') \right\|_2 
  +\left ( M_n + M_{|S|}\right) \|\th-\th' \|_2 \\ 
  \leq &  \left\|\H_S(\th')- \H_{[n]}(\th') \right\|_2 
  +\left ( M_n + M_{|S|}\right) \Delta.
}

Next, we will use a covering net argument to obtain a bound on the matrix empirical process. Note that similar bounds on the matrix forms can be obtained through other approaches like \emph{chaining} as well \cite{dicker2015flexible}. Let $\T_\Delta$ be a $\Delta$-net over the convex set $\C$. 
By the above inequality, we obtain

\eq{\label{eq::netIneq}
 \sup_{\th \in \C} \left\|\H_S(\th)- \H_{[n]}(\th) \right\|_2  
  \leq   \max_{\th' \in \T_\Delta}\left\|\H_S(\th')- \H_{[n]}(\th') \right\|_2 
  +\left ( M_n + M_{|S|}\right) \Delta.
}

Now we will argue that the right hand side is small 
with high probability using the matrix Hoeffding's inequality 
from \cite{tropp2012user}. 
By the union bound over $\T_\Delta$, we have

\eqn{
\P \left(
\max_{\th' \in \T_\Delta}\left\|\H_S(\th')- \H_{[n]}(\th') \right\|_2 > \e
\right)
\leq &
|\T_\Delta|\ 
\P \left(
\left\|\H_S(\th')- \H_{[n]}(\th') \right\|_2 > \e
\right).
}

For the first term on the right hand side, by Lemma \ref{lem::sphere}, we write:
\eqn{
|\T_\Delta| \leq \left ( \frac{\diam}{2\Delta/\sqrt{p}}\right)^p.
}

As before,  let $S = \{ s_1,s_2, ... ,s_{|S|}\}$, that is, $s_i$ 
denote the different indices in $S$.
For any $\th \in \C$ and $i=1,2,...,n$,
we define the centered Hessians $\W_i(\th)$ as
\eqn{
\W_i(\th) =\H_{s_i}(\th) - \H_{[n]}(\th).
}

By the Assumption (\ref{as::bound}), we have 
the same bounds as in Eq.(\ref{eq::boundsHoeff}).
Hence, for $\e >0$ and $\th \in \C$, 
by the matrix Hoeffding's inequality \cite{tropp2012user},
\eqn{
\P \left(
\left\|\H_S(\th)- \H_{[n]}(\th) \right\|_2 > \e
\right) \leq \ 
 2p \exp \left\{
-\frac{|S|\e^2}{32 K^2}
\right\}.
}
We would like to obtain an exponential decay with 
a rate of at least $\O(p)$. Hence, we require,

\eqn{
 p \log\left(\frac{\diam\sqrt{p}}{2\Delta}\right)+\log(2p) + p \leq & \ \ 
 p \log\left(\frac{4\diam\sqrt{p}}{\Delta}\right), \\ \leq&
\frac{|S|\e^2}{32 K^2},
}
which gives the optimal value of $\e$ as
\eqn{
\e \geq \sqrt{
\frac{32K^2p}{|S|} \log\left(\frac{4\diam\sqrt{p}}{\Delta}\right).
}
}
Therefore, we conclude that for the above choice of $\e$, 
with probability at least $1-e^{-p}$, we have 
\eqn{
\max_{\th \in \T_\Delta}\left\|\H_S(\th)- \H_{[n]}(\th) \right\|_2 < \sqrt{
\frac{32K^2p}{|S|} \log\left(\frac{4\diam\sqrt{p}}{\Delta}\right)
}.
}

Applying this result to the inequality in Eq.(\ref{eq::netIneq}), 
we obtain that with probability at least $1-e^{-p}$,

\eqn{
 \sup_{\th \in \C} \left\|\H_S(\th)- \H_{[n]}(\th) \right\|_2  
  \leq   \sqrt{
\frac{32K^2p}{|S|} \log\left(\frac{4\diam\sqrt{p}}{\Delta}\right)
}
  +\left ( M_n + M_{|S|}\right) \Delta.
}

The right hand side of the above inequality depends on the net 
covering diameter $\Delta$. We optimize over $\Delta$ using 
Lemma \ref{lem::epsilon} which provides for

\eqn{
\Delta = 4\sqrt{
\frac{K^2p}{\left ( M_n + M_{|S|}\right)^2|S|} 
\log\left ( \frac{ \diam^2 \left ( M_n + M_{|S|}\right)^2 |S|}{K^2}\right)
},
}

we obtain that with probability at least $1-e^{-p}$,

\eqn{
 \sup_{\th \in \C} \left\|\H_S(\th)- \H_{[n]}(\th) \right\|_2  
  \leq 
8K \sqrt{
  \frac{p}{|S|} 
\log\left (\frac{  \diam^2 \left ( M_n + M_{|S|}\right)^2 |S|}{K^2}\right)
}.
}
Combining this with the bound stated in Eq.(\ref{eq::mainBound}), 
we conclude the proof.
\end{proof}

\begin{proof}[Proof of Theorem \ref{thm::coefficients}]
\eqn{
\left | \xi_1^t - \xi_1^*\right| =&
\left | \frac{\lambda_p^t}{\lambda_{r+1}^t} - 
\frac{\lambda_p^*}{\lambda_{r+1}^*}\right| +
cK\sqrt{\frac{ \log(p)}{|S_t|}}\left | \frac{1}{\lambda_{r+1}^t} - 
\frac{1}{\lambda_{r+1}^*}\right| \\
\leq &\frac{K |\lambda_{r+1}^t-\lambda_{r+1}^*| + K|\lambda_{p}^t-\lambda_{p}^*| }
{\lambda_{r+1}^*\lambda_{r+1}^t}+
cK\sqrt{\frac{ \log(p)}{|S_t|}}\frac{|\lambda_{r+1}^t-\lambda_{r+1}^*|}{\lambda_{r+1}^*\lambda_{r+1}^t} 
}
By the Weyl's 
and matrix Hoeffding's \cite{tropp2012user} inequalities
(See Eq.~(\ref{eq::hoeffding}) for details), we can write 
\eqn{
|\lambda_{j}^t-\lambda_{j}^*| \leq \left \|\H_{S_t}(\hth^t) - \H_{[n]}(\th_*)\right \|_2 \leq 
cK \sqrt{\frac{\log(p)}{|S_t|}},
}
with probability $1-2/p$. Then,
\eqn{
\left | \xi_1^t - \xi_1^*\right|
\leq &
\frac{c'K \sqrt{\frac{\log(p)}{|S_t|}}}
{\lambda_{r+1}^*\lambda_{r+1}^t}+
\frac{c''K^2{\frac{ \log(p)}{|S_t|}}}{\lambda_{r+1}^*\lambda_{r+1}^t}, \\
\leq &
\frac{c'''K \sqrt{\frac{\log(p)}{|S_t|}}}
{k\left(k- cK\sqrt{\frac{\log(p)}{|S_t|}}\right) },
}
for some constants $c$ and $c'''$.
\end{proof}

\begin{proof}[Proof of Corollary \ref{cor::glm}]
Observe that 
$f_i(\th) =  \Phi (\<x_i,\th\>)-y_i\<x_i , \th \> $,  
and $\gradS_\th f_i(\th) = x_ix_i^T \ddphi(\<x_i,\th\>)$.
For an index set $S$, we have $\forall \th, \th' \in \C$
\eqn{
\left \|\H_{S}(\th) - \H_{S}(\th')\right\|_2 
=& \left\|\frac{1}{|S|}\sum_{i\in S} x_i x_i^T \left[ \ddphi (\<x_i,\th\>) - \ddphi (\<x_i,\th'\>)\right ]\right \|_2,\\
&\leq  L \max_{i \in S}\|x_i\|_2^3 \ \| \th - \th' \|_2 
\leq  L R_x^{3/2} \ \| \th - \th' \|_2.
}
Therefore, the Assumption \ref{as::Lipschitz} is satisfied with 
the Lipschitz constant $M_{|S_t|} \coloneqq L R_x^{3/2}.$
Moreover, by the inequality
\eqn{
\left \| \gradS_\th f_i(\th)  \right \|_2 = \| x_i \|_2^2 \ \ddphi(\<x_i,\th\>)\leq  R_x,= \left \| x_ix_i^T \ddphi(\<x_i,\th\>) \right \|_2
}
the Assumption \ref{as::bound} is satisfied for 
$
K \coloneqq  R_x .
$
We conclude the proof by applying Theorem \ref{thm::uniformSampling}. 
\end{proof}

\section{Properties of composite convergence}\label{sec::theoryExtra}

In the previous sections, we showed that \ALG gets a composite convergence rate, i.e., 
the $\ell_2$ distance from the current iterate to the optimal value can be 
bounded by the sum of a linearly and a quadratically converging term.
We study such convergence rates assuming the coefficients do not change at each iteration $t$. 
Denote by $\Delta_t$, the aforementioned $\ell_2$ distance at iteration step $t$, i.e.,
\eq{\label{eq::composite}
\Delta_t = \|\hth^t - \theta_* \|_2,
}
and assume that the algorithm gets a composite convergence rate as
\eqn{
\forall t\geq 0,\ \ \ \ \ \ \ \Delta_{t+1} \leq \xi_1 \Delta_{t} + \xi_2 \Delta_{t}^2,
}
where $\xi_1,\xi_2>0$ denote the coefficients of linearly and quadratically converging terms, respectively.

\subsection{Local asymptotic rate}
We state the following theorem on the local convergence 
properties of compositely converging algorithms.

\begin{lemma}\label{lem::localConvergence}
For a compositely converging algorithm as in Eq.~(\ref{eq::composite})
with coefficients $1>\xi_1,\xi_2>0$, if the initial distance $\Delta_0$ satisfies 
$\Delta_0 < (1-\xi_1)/\xi_2$, then we have
\eqn{
\limsup_{t\to \infty} - \frac{1}{t}\log(\Delta_t) \leq - \log(\xi_1).
}
\end{lemma}
The above theorem states that the local convergence of a compositely 
converging algorithm will be dominated by the linear term.

\begin{proof}[Proof of Lemma \ref{lem::localConvergence}]
The condition on the initial point implies that $\Delta_t \to 0$ as $t \to \infty$. 
Hence, for any given $\delta >0 $, there exists a positive integer $T$ such that 
$\forall t \geq T$, we have $\Delta_t < \delta/\xi_2$. For such values of $t$, we write
\eqn{
\xi_1 + \xi_2 \Delta_t < \xi_1 + \delta,
}
and using this inequality we obtain
\eqn{
\Delta_{t+1} <
 (\xi_1  + \delta) \Delta_{t} .
}
The convergence of above recursion gives
\eqn{
-\frac{1}{t}\log(\Delta_t) < -\log(\xi_1  +  \delta)-\frac{1}{t}\log(\Delta_0).
}
Taking the limit on both sides concludes the proof.
\end{proof}

\subsection{Number of iterations}
The total number of iterations, combined with the per-iteration cost, 
determines the total complexity of an algorithm. 
Therefore, it is important to derive an upper bound on 
the total number of iterations of a compositely converging algorithm.

\begin{lemma}\label{lem::numIterations}
For a compositely converging algorithm as in Eq.~(\ref{eq::composite})
with coefficients $\xi_1,\xi_2\in(0,1)$, 
assume that the initial distance $\Delta_0$ satisfies 
$\Delta_0 < (1-\xi_1)/\xi_2$ and for a given tolerance $\epsilon$, define the interval
$$D = \left(\max\left\{\epsilon,\frac{\xi_1\Delta_0}{1-\xi_2\Delta_0}\right\},\Delta_0\right).$$
Then the total number of iterations needed to 
approximate the true minimizer with $\epsilon$ tolerance is upper bounded by $T(\delta_*)$, where
\eqn{
\delta_* = \argmin_{\delta \in D} T(\delta)
}
and
\eqn{
T(\delta) = \log_2\left( \frac{\log\left(\xi_1  + \delta\xi_2\right)}{\log\left(\frac{\Delta_0}{\delta}(\xi_1  + \delta\xi_2)\right)} \right)+
\frac{\log\left(\frac{\epsilon}{\delta}\right)}{\log(\xi_1+ \xi_2\delta)}.
}
\end{lemma}

\begin{proof}[Proof of Lemma \ref{lem::numIterations}]
We have $\Delta_t \to 0$ as $t \to \infty$ by the condition on initial point $\Delta_0$. Let 
$\delta \in D$
be a real number and $t_1$ be the last iteration step such that $\Delta_t > \delta$. Then $\forall t \geq t_1$,
\eqn{
\Delta_{t+1} \leq & \xi_1 \Delta_{t} + \xi_2 \Delta_{t}^2,\\
\leq & \left(\frac{\xi_1}{\delta}  + \xi_2\right) \Delta_{t}^2.\\
}
Therefore, in this regime, the convergence rate of the algorithm is dominated by a 
quadratically converging term with coefficient $(\xi_1/\delta + \xi_2)$. The total number of 
iterations needed to attain a tolerance of $\delta$ is upper bounded by
\eqn{
t_1 \leq \log_2\left( \frac{\log\left(\xi_1  + \delta\xi_2\right)}{\log\left(\frac{\Delta_0}{\delta}(\xi_1  + \delta\xi_2)\right)} \right).
}
When $\Delta_t <\delta$, namely $t>t_1$, we have
\eqn{
\Delta_{t+1} \leq & \xi_1 \Delta_{t} + \xi_2 \Delta_{t}^2,\\
\leq & \left(\xi_1+ \xi_2\delta\right) \Delta_{t}.\\
}
In this regime, the convergence rate is dominated by a linearly converging term with 
coefficient $\left(\xi_1+ \xi_2\delta\right)$. Therefore, the total number of 
iterations since $t_1$ until a tolerance of $\epsilon$ is reached can be upper bounded by
\eqn{
t_2 \leq \frac{\log\left(\frac{\epsilon}{\delta}\right)}{\log(\xi_1+ \xi_2\delta)}.
}
Hence, the total number of iterations needed for a composite algorithm 
as in Eq.~\ref{eq::composite} to reach a tolerance of $\epsilon$ is upper bounded by
\eqn{
T(\delta) =t_1+t_2 =  \log_2\left( \frac{\log\left(\xi_1  + \delta\xi_2\right)}{\log\left(\frac{\Delta_0}{\delta}(\xi_1  + \delta\xi_2)\right)} \right)+
\frac{\log\left(\frac{\epsilon}{\delta}\right)}{\log(\xi_1+ \xi_2\delta)}.
}
The above statement holds for any $\delta \in D$. Therefore, we minimize $T(\delta)$ over the set $D$.
\end{proof}

\section{Choosing the step size $\step_t$}\label{sec::stepsize}

In most optimization algorithms, step size plays a crucial role.  
If the dataset is so large that one cannot try out many values of the step size. 
In this section, 
we describe an efficient and adaptive way for this purpose by using the 
theoretical results derived in the previous sections.

In the proof of Lemma \ref{lem::uniformSampling}, we observe that the 
convergence rate of \ALG is governed by the term
\eqn{
\left\| \I - \step_t \Q^t H_{[n]} (\tth)\right \|_2 \leq \left\| \I - \step_t \Q^t H_{[n]} (\hth^t)\right \|_2 
+ \step_t\left\| \Q^t  \right \|_2  \left\| H_{[n]} (\hth^t) - H_{[n]} (\tth)\right \|_2
}
where $\Q^t$ is defined as in Algorithm 1. The right hand side of the above equality 
has a linear dependence on $\step_t$. We will see later that this term has no 
effect in choosing the right step size. On the other hand, the first term on 
the right hand size can be written as,
\eqn{
\left\| \I - \step_t \Q^t H_{[n]} (\hth^t)\right \|_2  = &
\max\left\{ 
1-\step_t \lmin(\Q^t H_{[n]} (\hth^t)),  
\step_t \lmax(\Q^t H_{[n]} (\hth^t)) -1
\right\}.
}

If we optimize the above quantity over 
$\step_t$, we obtain the optimal step size as
\eq{\label{eq::optimalGamma}
\step_t = \frac{2}{
\lmin(\Q^t H_{[n]} (\hth^t))+\lmax(\Q^t H_{[n]} (\hth^t))
}.
}
It is worth mentioning that for the Newton's method where $\Q^t = H_{[n]} (\hth^t)^{-1}$, the above quantity is equal to 1.

Since \ALG does not compute the full Hessian $\H_{[n]} (\hth^t)$ (which would take $\O(np^2)$ computation), we will relate the quantity in Eq.~(\ref{eq::optimalGamma}) to the first few eigenvalues of $\Q^t$. Therefore, our goal is to relate the eigenvalues of $\Q^t H_{[n]} (\hth^t)$ to that of $\Q^t$.

By the Lipschitz continuity of eigenvalues 
, we write
\eq{\label{eq::gammaBound1}
\left |1 -  \lmax(\Q^t H_{[n]} (\hth^t))\right| \leq &\left\|\Q^t\right\|_2 \left\| H_{S} (\hth^t)-H_{[n]} (\hth^t)\right\|_2,\nonumber\\
= & \frac{1}{\lambda^t_{r+1}}\O \left (\sqrt{\frac{\log(p)}{|S|}}\right).
}

Similarly, for the minimum eigenvalue, we can write
\eq{\label{eq::gammaBound2}
\left |\frac{\lambda^t_p}{\lambda^t_{r+1}} -  \lmin(\Q^t H_{[n]} (\hth^t))\right| \leq \frac{1}{\lambda_{r+1}}\O \left (\sqrt{\frac{\log(p)}{|S|}}\right).
}

One might be temped to use 1 and $\lambda^t_p/\lambda^t_{r+1}$ for the minimum and the maximum eigenvalues of $\Q^tH_{[n]} (\hth^t)$, but the optimal values might be slightly different from these values if the sample size is chosen to be small. On the other hand, the eigenvalues $\lambda^t_{r+1}$ and $\lambda^t_p$ can be computed with $\O(p^2)$ cost and we already know the order of the error term. That is, one can calculate $\lambda^t_{r+1}$ and $\lambda^t_p$ and use the error bounds to correct the estimate. 

The eigenvalues of the sample covariance matrix will concentrate around the true values, spreading to be larger for large eigenvalues and smaller for the small eigenvalues. That is, if we will we will overestimate if we estimate $\lambda_1$ with $\lambda_1^t$. Therefore, if we use 1, we will always underestimate the value of $\lmax(\Q^t H_{[n]} (\hth^t))$, which, based on Eq.~(\ref{eq::gammaBound1}) and Eq.~(\ref{eq::gammaBound2}), suggests a correction term of 
$\O \left (\sqrt{{\log(p)}/{|S|}}\right)$. Further, the top $r+1$ eigenvalues of $[Q^t]^{-1}$ are close to the eigenvalues of $H_{[n]} (\hth^t)$, but shifted upwards if $p/2 >r$. When $p/2 <r$, we see an opposite behavior. Hence, we add or subtract a correction term of order $\O \left (\sqrt{{\log(p)}/{|S|}}\right)$ to $\lambda^t_p/\lambda^t_{r+1}$ whether $p/2 > r$ or $p/2 < r$, respectively. The corrected estimators could be written as

\eqn{
&\widehat{\lmax}\left(\Q^t H_{[n]} (\hth^t)\right) = 1 + \O \left (\sqrt{\frac{\log(p)}{|S|}}\right),\\
&\widehat{\lmin}\left(\Q^t H_{[n]} (\hth^t)\right) = \frac{\lambda_p}{\lambda_{r+1}} + \O \left (\sqrt{\frac{\log(p)}{|S|}}\right)
\ \ \ \ \ \text{if $p/2>r$},\\
&\ \ \ \ \ \ \ \ \ \ \ \ \ \ \ \ \ \ \ \ \ \  = \frac{\lambda_p}{\lambda_{r+1}} - \O \left (\sqrt{\frac{\log(p)}{|S|}}\right)
\ \ \ \ \ \text{if $p/2< r$}.
}

We are more interested in the case where $p/2 >r$. In this case, we suggest the step size for the iteration step $t$ as

\eqn{
\step_t = \frac{2}{
1 + \frac{\lambda^t_p}{\lambda^t_{r+1}} +\O \left (\sqrt{\frac{\log(p)}{|S|}}\right)
}
}
which uses the eigenvalues that are already computed to construct $\Q^t $. Contrary to the most algorithms, the optimal step size of \ALG is generally larger than 1.

\section{Further experiments and details}\label{sec::furtherExperiments}

\begin{figure}[t]
\centering
  \includegraphics[width=6.5in]{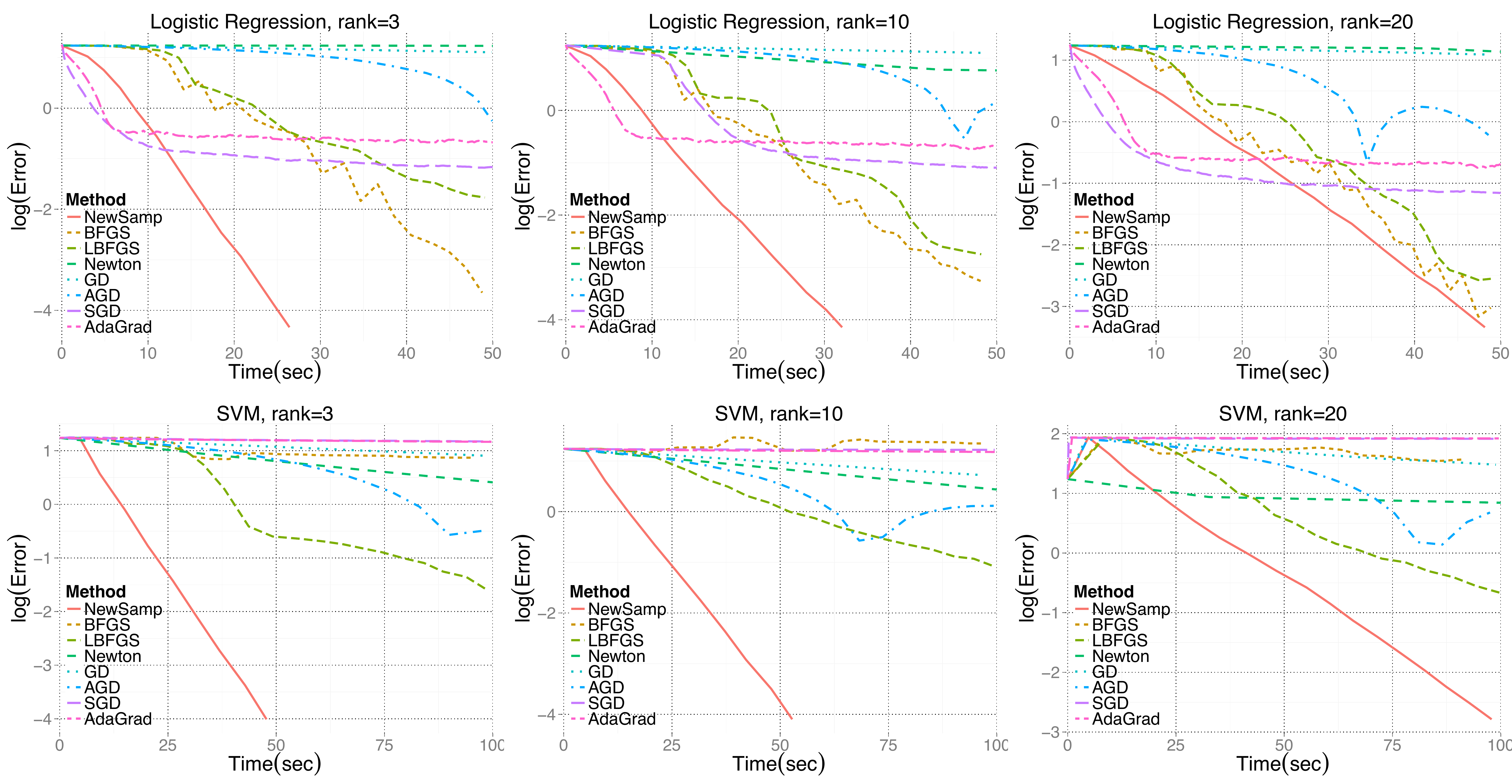}
  \caption{\label{fig::plot2}
The plots demonstrate the behavior of several optimization methods on a synthetic data set for training SVMs. The elapsed time in seconds versus $\log$ of $\ell_2$-distance to the true minimizer is plotted. Red color represents the proposed method \ALG.
}
\end{figure}

In this section, we present the details of the experiments presented in Figure \ref{fig::plot} 
and provide additional simulation results.

We first start with additional experiments. 
The goal of this experiment is to further analyze the effect of rank in the performance of \ALG. 
We experimented using $r$-spiked model for $r=3,10,20$. The case $r=3$ was already presented in Figure \ref{fig::plot}, which is included in Figure \ref{fig::plot2} to ease the comparison.
The results are presented in Figures \ref{fig::plot2} and the details are summarized in Table \ref{tab::simdetails}. In the case of LR optimization, we observe through Figure \ref{fig::plot2} that stochastic algorithms enjoy fast convergence in the beginning but slows down later as they get close to the true minimizer. The algorithms that come closer to \ALG in terms of performance are BFGS and LBFGS. Especially when $r=20$, performance of BFGS and that of \ALG are similar, yet \ALG still does better. 
In the case of SVM optimization, the algorithm that comes closer to \ALG is Newton's method. 

We further demonstrate how the algorithm coefficients $\xi_1$ and $\xi_2$ between datasets in Figure \ref{fig::xiChange}.
\begin{table}
\centering
\begin{tabular}{|l|ll|ll|lr|}
\multicolumn{5}{c}{\hspace{1.1in}Logistic Regression}\\
\hline
    & \multicolumn{2}{c}{Rank=3}  \vline& \multicolumn{2}{c}{Rank=10}\vline& \multicolumn{2}{c}{Rank=20}\vline\\
\hline
Method    & Elapsed(sec) & Iter 	& Elapsed(sec)   &   Iter& Elapsed(sec)   &   Iter\\
\hline
\ALG        &26.412      &  12& 	32.059 	& 15		& 55.995	& 26\\
BFGS      &50.699      & 22& 	54.756 	& 31		& 56.606	& 34 \\
LBFGS      &103.590      &47  & 64.617 	& 37		& 107.708	& 67 \\
Newton   &18235.842     &  449&  	35533.516 	& 941		& 31032.893	& 777 \\
GD          &345.025   & 198	& 322.671 & 198 	& 311.946 & 197\\
AGD          &449.724   & 233	& 436.282 & 272 	& 450.734 & 290\\
\hline
\end{tabular}\\

\vspace{.1in}

\begin{tabular}{|l|ll|ll|lr|}
\multicolumn{5}{c}{\hspace{1.1in}Support Vector Machines}\\
\hline
    & \multicolumn{2}{c}{Rank=3}  \vline& \multicolumn{2}{c}{Rank=10}\vline& \multicolumn{2}{c}{Rank=20}\vline\\
\hline
Method    & Elapsed(sec) & Iter 	& Elapsed(sec)   &   Iter& Elapsed(sec)   &   Iter\\
\hline
\ALG        &47.755      &  8& 	52.767 	& 9		& 124.989	& 22\\
BFGS      &13352.254      & 2439& 	10672.657 	& 2219		& 21874.637	& 4290 \\
LBFGS      &326.526      & 67 & 218.706 	& 44		& 275.991	& 55 \\
Newton   &775.191     &  16&  	734.480 	& 16		& 4159.486	& 106 \\
GD          &1512.305   & 238	& 1089.413 & 237 	& 1518.063 & 269\\
AGD          &1695.44   & 239	& 1066.484 & 238 	& 1874.75 & 294\\
\hline
\end{tabular}
\caption{\label{tab::simdetails}
Details of the simulations presented in Figures \ref{fig::plot2}.
}
\end{table}

\begin{table}
\centering
\begin{tabular}{|l|l|l|l|r|}
\multicolumn{5}{c}{CT Slices Dataset}\\
\hline
    & \multicolumn{2}{c}{LR} \vline  & \multicolumn{2}{c}{SVM}\vline\\
\hline
Method    & Elapsed(sec) & Iter		& Elapsed(sec) & Iter\\
\hline
\ALG        &9.488      &  19& 	22.228 	& 33\\
BFGS      &9.568      & 38& 	2094.330 & 5668 \\
LBFGS      &51.919      & 217 & 165.261 	& 467		 \\
Newton   &14.162     &  5&  	58.562 	& 25		 \\
GD          &350.863   & 2317	& 1660.190 & 4828 	 \\
AGD          &176.302   & 915	& 1221.392 & 3635 	\\
\hline
\end{tabular}
%
%
\begin{tabular}{|l|l|l|l|r|}
\multicolumn{5}{c}{MSD Dataset}\\
\hline
    & \multicolumn{2}{c}{LR} \vline  & \multicolumn{2}{c}{SVM}\vline\\
\hline
Method    & Elapsed(sec) & Iter		& Elapsed(sec) & Iter\\
\hline
\ALG        &25.770      &  38& 	71.755 	& 49\\
BFGS      &43.537      & 75& 	9063.971 & 6317 \\
LBFGS      &81.835      & 143 & 429.957 	& 301		 \\
Newton   &144.121     &  30&  	100.375 	& 18		 \\
GD          &642.523   & 1129	& 2875.719 & 1847 	 \\
AGD          &397.912   & 701	& 1327.913 & 876 	\\
\hline
\end{tabular}\\
%
%
\begin{tabular}{|l|l|l|l|r|}
\multicolumn{5}{c}{Synthetic Dataset}\\
\hline
    & \multicolumn{2}{c}{LR} \vline  & \multicolumn{2}{c}{SVM}\vline\\
\hline
Method    & Elapsed(sec) & Iter		& Elapsed(sec) & Iter\\
\hline
\ALG        &26.412      &  12& 	47.755 	& 8\\
BFGS      &50.699      & 22& 	13352.254 & 2439 \\
LBFGS      &103.590      & 47 & 326.526 	& 67		 \\
Newton   &18235.842     &  449&  	775.191 	& 16		 \\
GD          &345.025   & 198	& 1512.305 & 238 	 \\
AGD          &449.724   & 233	& 1695.44 & 239 	\\
\hline
\end{tabular}\\
\caption{\label{tab::details}
Details of the experiments presented in Figure \ref{fig::plot}.}
\end{table}

\begin{figure}[b]
\centering
  \includegraphics[width=3.2in]{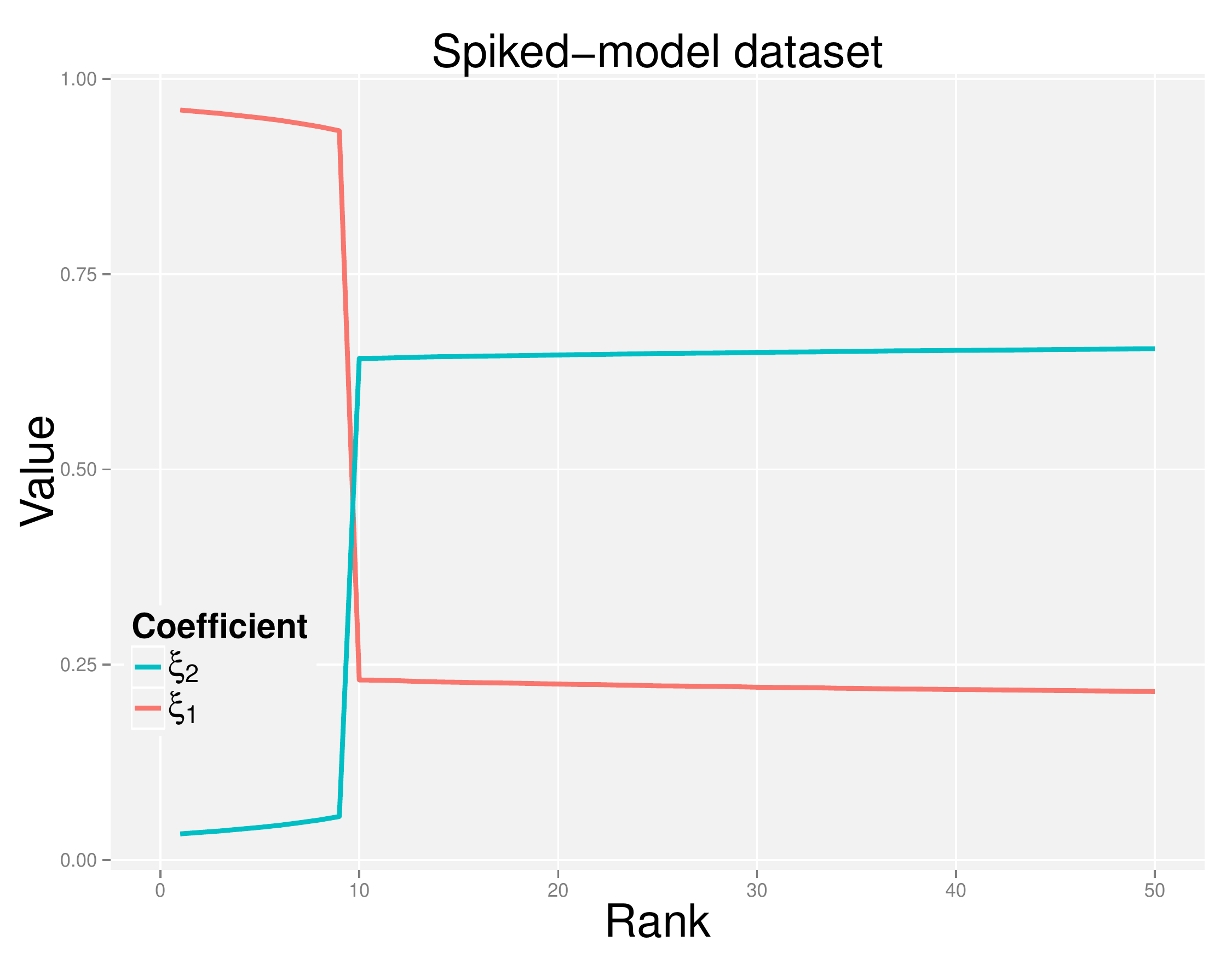}
\includegraphics[width=3.2in]{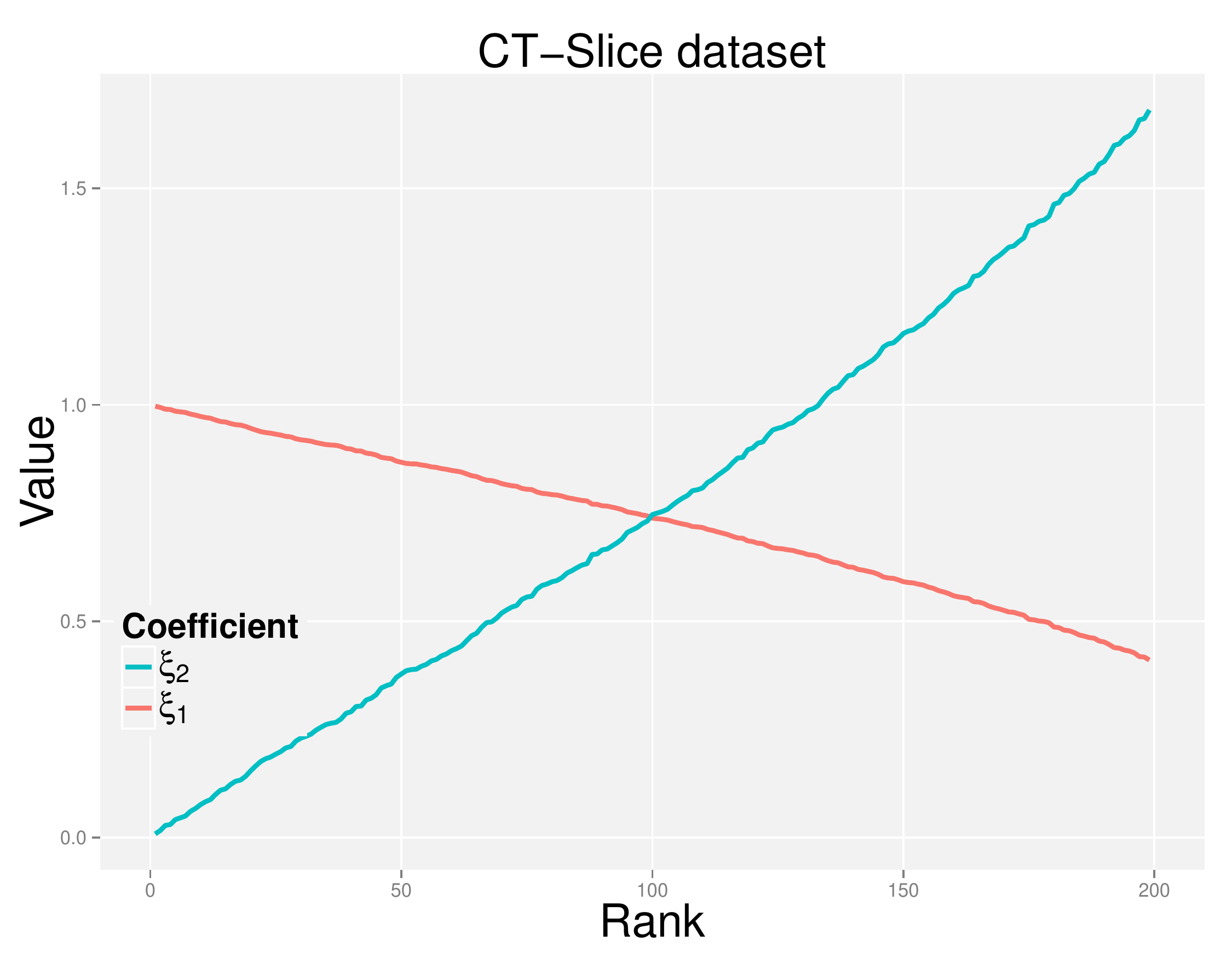}
\includegraphics[width=3.2in]{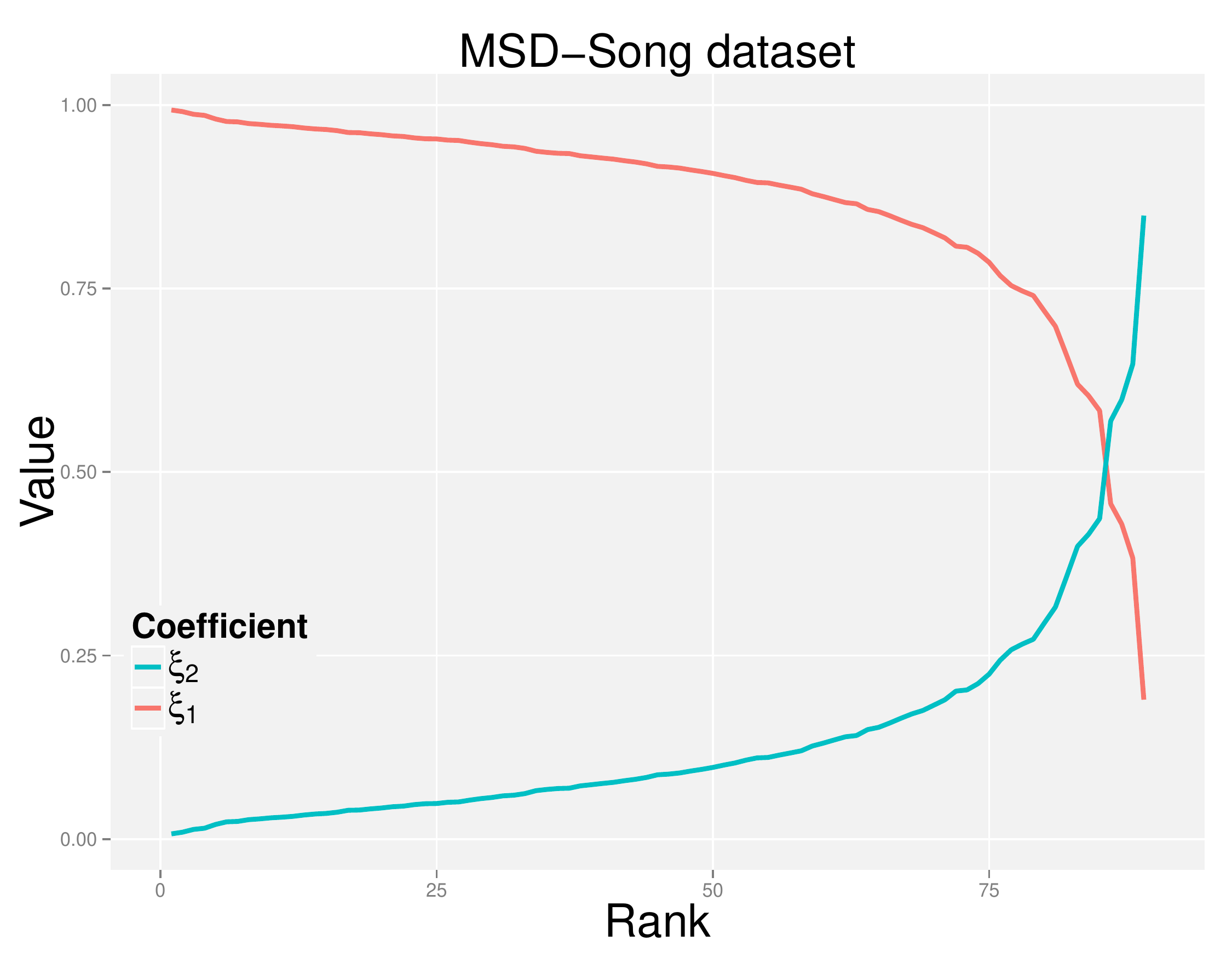}
\includegraphics[width=3.2in]{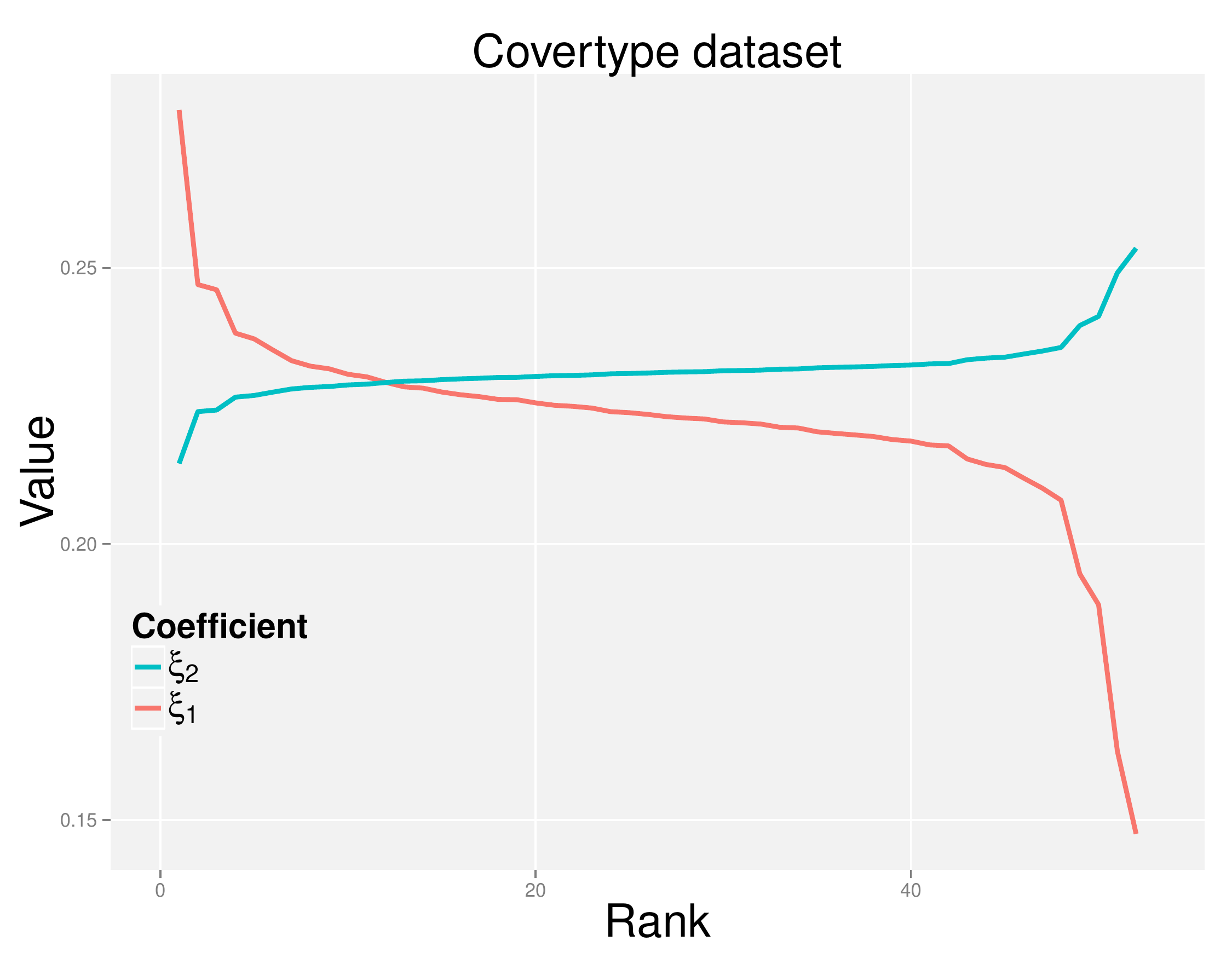}
  \caption{\label{fig::xiChange}
The plots demonstrate the behavior of $\xi_1$ and $\xi_2$ over several datasets.
}
\end{figure}

\section{Useful lemmas} \label{sec::auxiliary}

\begin{lemma}\label{lem::sphere}
Let $\C$ be convex and bounded set in $\reals^p$ and $T_\epsilon$ be an $\epsilon$-net over $\C$. Then,
\eqn{
|T_\epsilon|\leq \left(\frac{\diam}{2\epsilon/\sqrt{p}}\right)^p.
}
\end{lemma}
\begin{proof}[Proof of Lemma \ref{lem::sphere}]
A similar proof appears in \cite{van1996weak}. The set $\C$ can be contained in a $p$-dimensional cube of size $\diam$. Consider a grid over this cube with mesh width $2\epsilon/\sqrt{p}$. Then $\C$ can be covered with at most $(\diam/(2\epsilon/\sqrt{p}))^p$ many cubes of edge length $2\epsilon/\sqrt{p}$. If ones takes the projection of the centers of such cubes onto $\C$ and considers the circumscribed balls of radius $\epsilon$, we may conclude that $\C$ can be covered with at most
$$\left(\frac{\diam}{2\epsilon/\sqrt{p}}\right)^p$$
many balls of radius $\epsilon$.
\end{proof}

\begin{lemma}[\cite{vershynin2010introduction}]\label{lem::epsNet}
Let $X$ be a symmetric $p\times p$ matrix, and let $T_\epsilon$ be an $\epsilon$-net over $S^{p-1}$. Then,
\eqn{
\| X\|_2 \leq\frac{1}{1-2\epsilon}\ \sup_{v\in T_\epsilon}\left|\<Xv,v\>\right|.
}
\end{lemma}


\begin{lemma}[\cite{gross2010note}] \label{lem::op-Bernstein}
Let $\mathcal{X}$ be a finite set of Hermitian matrices in $\reals^{p \times p} $ where $\forall X_i \in \mathcal{X}$, we have
\eqn{
\E[X_i] =& 0, \ \ \ \ \ \ 
\left\|X_i \right\|_2 \leq& \gamma, \ \ \ \ \ \ \ 
\left\| \E [ X_i^2 ] \right\|_2 \leq& \sigma^2.
}
Given its size, let $S$ denote a uniformly random sample from $\{1,2,...,|\mathcal{X}|\}$
with or without replacement. Then we have
\eqn{
\P \left( \bigg\| \frac{1}{|S|}\sum_{i\in S}X_i \bigg\|_2 > \e \right) \leq 
2 p \exp\left \{ -|S|\min\left( \frac{\e^2}{4\sigma^2}, \frac{\e}{2\gamma}\right)\right\}.
}
\end{lemma}

\vspace{.2in}
\begin{lemma}\label{lem::darthvader}
Let $Z$ be a random variable with a density function $f$ and cumulative distribution function $F$. If $F^C=1-F$, then,
\eqn{
\left|\E [Z\ind_{\{ |Z|>t\}}]\right|\leq t \P(|Z|>t) + \int^\infty_t\P(|Z|>z)dz.
}
\end{lemma}
\begin{proof}
We write,
\eqn{
\E [Z\ind_{\{ |Z|>t\}}] = \int_t^\infty zf(z)dz + \int^{-t}_{-\infty}zf(z)dz.
}
Using integration by parts, we obtain
\eqn{
\int zf(z)dz =& -zF^C (z) + \int F^C(z)dz,\\
=& zF (z) - \int F(z)dz.
}
Since $\lim_{z\to\infty}zF^C(z)=\lim_{z\to-\infty}zF(z)=0$, we have
\eqn{
 \int_t^\infty zf(z)dz =& t F^C(t) + \int^\infty_tF^C(z)dz,\\
  \int_{-\infty}^{-t} zf(z)dz =& -t F(-t) - \int_{-\infty}^{-t}F(z)dz,\\
  =& -t F(-t) - \int_{t}^{\infty}F(-z)dz.
}
Hence, we obtain the following bound,
\eqn{
\left|\E [Z\ind_{\{ |Z|>t\}}]\right| =& \left|t F^C(t) + \int^\infty_tF^C(z)dz -t F(-t) - \int_{t}^{\infty}F(-z)dz\right|,\\
\leq& t\left( F^C(t) +  F(-t)\right) + \left(\int^\infty_tF^C(z)+F(-z)dz\right),\\
\leq& t \P(|Z|>t) + \int^\infty_t\P(|Z|>z)dz.
}
\end{proof}

\begin{lemma}\label{lem::epsilon}
For $a,b>0$, and $\epsilon$ satisfying 
\eqn{
\e = \left\{\frac{a}{2}\log\left(\frac{2b^2}{a}\right)\right\}^{1/2}\ \ \ \ \text{and }\ \ \ \ \frac{2}{a}b^2 > e,
}
we have 
$\e^2 \geq a\log(b/\e)$.
\end{lemma}

\begin{proof}
Since $a,b>0$ and $x\to e^x$ is a monotone increasing function,
the above inequality condition is equivalent to
\eqn{
\frac{2\e^2}{a} e^{\frac{2\e^2}{a}} \geq \frac{2b^2}{a}.
}
Now, we define the function $f(w)= we^w$ for $w>0$. $f$ is continuous and invertible on $[0,\infty)$. 
Note that $f^{-1}$ is also a continuous and increasing function for $w>0$. Therefore, we have
\eqn{
\e^2 \geq \frac{a}{2}f^{-1}\left(\frac{2b^2}{a}\right)
}
Observe that the smallest possible value for $\e$ would be simply the square root of ${a}f^{-1}\left({2b^2}/{a}\right)/{2}$. For simplicity, we will obtain a more interpretable expression for $\e$. By the definition of $f^{-1}$, we have
\eqn{
\log(f^{-1}(y)) + f^{-1}(y) = \log (y).
}
Since the condition on $a$ and $b$ enforces $f^{-1}(y)$ to be larger than 1, we obtain the simple inequality that 
\eqn{
f^{-1}(y)  \leq \log (y).
}
Using the above inequality, if $\e$ satisfies 
\eqn{
\e^2 = \frac{a}{2}\log\left(\frac{2b^2}{a}\right)\geq \frac{a}{2}g\left(\frac{2b^2}{a}\right),
}
we obtain the desired inequality.
\end{proof}

\end{document}